\documentclass{article}

\usepackage[preprint]{neurips_2026}

\usepackage[utf8]{inputenc} 
\usepackage[T1]{fontenc}    
\usepackage{hyperref}       
\usepackage{url}            
\usepackage{booktabs}       
\usepackage{amsfonts}       
\usepackage{nicefrac}       
\usepackage{microtype}      
\usepackage{xcolor}         
\usepackage{graphicx}
\usepackage{amsthm}
\newtheorem{proposition}{Proposition}
\newtheorem{remark}{Remark}
\newtheorem{corollary}{Corollary}

\PassOptionsToPackage{numbers}{natbib}

\title{Patch-PODiff-ViT: Structured Latent Diffusion with Patchwise POD for Super-Resolution and Uncertainty Quantification}

\author{%
  Onkar Jadhav \\
  School of Earth and Oceans, UWA Oceans Institute\\
  University of Western Australia\\
  Crawley, WA, Australia \\
  \texttt{onkar.jadhav@uwa.edu.au} \\
  \And
  Tim French \\
  School of Physics, Mathematics and Computing, Computer Science and Software Engineering \\
  University of Western Australia\\
  Crawley, WA, Australia \\
  \AND
  Matthew Rayson \\
  School of Earth and Oceans, UWA Oceans Institute\\
  University of Western Australia\\
  Crawley, WA, Australia \\
  \And
  Nicole L. Jones \\
  School of Earth and Oceans, UWA Oceans Institute\\
  University of Western Australia\\
  Crawley, WA, Australia \\
}

\begin{document}

\maketitle

\begin{abstract}
Diffusion models enable probabilistic super-resolution and conditional generation, but pixel-space methods are computationally expensive and learned latent spaces often lack interpretable uncertainty quantification. 
We introduce Patch-PODiff-ViT, a structured latent diffusion framework in which the latent space is defined by patchwise Proper Orthogonal Decomposition (POD), a fixed linear orthonormal basis over local patches, rather than learned by a nonlinear autoencoder. This yields low-dimensional, variance-ordered tokens that preserve spatial structure and enable efficient diffusion in a structured low-dimensional latent space with a Vision Transformer. Because the decoder is fixed, linear, and orthonormal, latent coefficient uncertainty can be propagated directly to physical-space predictive variance, enabling analytic propagation of predictive variance through the linear decoder without Monte Carlo estimation in pixel space. Across sea surface temperature, medical imaging, and natural images, the method achieves strong reconstruction with fewer parameters and lower memory, while producing well-calibrated spatial uncertainty that closely matches empirical ensembles.
\end{abstract}

\section{Introduction}
High-resolution spatial fields arise across a wide range of applications, including climate modeling \citep{price2023gencast, watt2024generative,Jadhav2025RCNN}, medical imaging \citep{moser2024diffusion}, and natural image generation \citep{rombach2022high}. While modern pipelines provide coarse observations at scale, resolving fine-scale spatial structure remains computationally expensive. Super-resolution methods address this challenge \citep{leinonen2020stochastic, stengel2020adversarial, saharia2022image, leinonen2023latent}, but accurate reconstruction alone is often insufficient. Reliable uncertainty quantification is essential for downstream decision-making, particularly in regimes with sharp gradients, localized extremes, and incomplete observations.

Diffusion-based models \citep{ho2020denoising,song2020score} provide a powerful framework for probabilistic super-resolution and conditional generation \citep{saharia2022image}, producing high-fidelity samples and enabling ensemble-based uncertainty estimation. However, operating in pixel space is computationally expensive at high resolutions, leading to large models, high memory usage, and slow sampling, which makes ensemble generation costly in practice \citep{ho2020denoising}.

Latent diffusion models mitigate this cost by operating in a lower-dimensional space learned via autoencoders \citep{rombach2022high, vahdat2021score}. While effective for natural images, these latent representations are typically nonlinear and lack a direct correspondence to spatial structure, making uncertainty propagation to physical-space predictive variance less interpretable \citep{bohm2019uncertainty} and limiting their applicability in settings where structure and interpretability are important.
In contrast, many spatial fields exhibit strong local structure that can be efficiently represented using linear reduced-order methods such as Proper Orthogonal Decomposition (POD) \citep{sirovich1987turbulence, berkooz1993proper,benner2015survey}. POD yields an orthonormal, variance-ordered basis that captures dominant spatial patterns and defines a geometrically meaningful latent space, where coefficients correspond to progressively finer scales. Importantly, this structure is often more pronounced at the local level: individual $p \times p$ patches can be represented using far fewer modes than the full field while retaining the same variance, enabling substantial compression. Such local structure is not limited to physical systems, and is often observed in medical and natural images, particularly at the patch level.

Despite its widespread use in scientific computing \citep{coscia2024generative, du2024conditional}, POD remains underexplored as a structured latent space for diffusion-based generative modeling. Although it provides a linear and interpretable mapping between latent coefficients and spatial fields, this structure has not been fully leveraged for efficient and principled uncertainty propagation.

In this work, we introduce Patch-PODiff-ViT, an extension of PODiff \citep{jadhav2026podiff}, a conditional diffusion framework operating in a structured latent space defined by patchwise POD representations. Instead of learning a latent space, we construct a fixed, variance-ordered basis over local patches, yielding low-dimensional tokens that preserve spatial locality and scale separation. Diffusion is performed over these tokens using a Vision Transformer denoiser, enabling global spatial reasoning with improved efficiency. Crucially, the linear POD structure allows predictive uncertainty in latent space to be propagated analytically to the physical domain, providing a tractable and interpretable alternative to explicit full-resolution covariance estimation without additional learned components.

This formulation connects reduced-order and generative modeling, showing that structured linear representations combined with expressive denoisers enable efficient and interpretable probabilistic inference. Unlike pixel-space diffusion, it scales to high-resolution fields with lower computational cost, and unlike learned latent diffusion, it preserves a direct and tractable link between latent variables and spatial statistics.

We evaluate the method on three domains: sea surface temperature (SST) downscaling, medical image super-resolution, and natural image reconstruction. Across all datasets, Patch-PODiff-ViT achieves strong reconstruction performance with fewer parameters and lower memory than pixel-space diffusion, while producing well-calibrated, spatially meaningful uncertainty estimates that closely match empirical ensemble statistics.



    
    

\paragraph{Contributions.}
This paper makes four contributions.
(i) We introduce Patch-PODiff-ViT, a structured latent diffusion framework using patchwise POD and transformer-based denoising.
(ii) We provide a theoretical foundation: Proposition~\ref{prop:patch_error} bounds reconstruction error under variance truncation, and Proposition~\ref{prop:patch_uq} enables closed-form propagation of latent uncertainty to pixel space under a block-diagonal approximation.
(iii) We show improved computational efficiency via a low-dimensional structured latent representation.
(iv) We validate the method across geophysical, medical, and natural image datasets, achieving strong reconstruction and well-calibrated uncertainty.

\section{Method}
\label{sec:method}
Patch-PODiff-ViT performs conditional generative modeling in a structured latent space defined by patchwise POD, reducing fields from $H \times W$ pixels to $P \times K$ latent tokens (Figure~\ref{fig:pipeline}). Each image is decomposed into $P$ patches, projected onto a shared POD basis $\Phi$, and denoised by a Vision Transformer in token space. For super-resolution, conditioning and target fields are encoded in the same latent space. At inference, denoised coefficients are decoded linearly via $\hat{\mathbf{u}}_p = \bar{\mathbf{u}} + \Phi\hat{\mathbf{a}}_p$ and stitched to reconstruct the field, enabling analytic uncertainty propagation through the POD decoder.
\subsection{Patchwise POD Representation}
\label{sec:pod}
Let $\{\mathbf{u}_i\}_{i=1}^{N}$ denote a set of high-resolution training fields with $\mathbf{u}_i \in \mathbb{R}^{H \times W \times C}$. Each field is decomposed into patches of size $p \times p$, extracted with stride $r \leq p$, yielding
\[
    \mathbf{u}_i = \{\mathbf{u}_{i,p}\}_{p=1}^{P}, \quad
    \mathbf{u}_{i,p} \in \mathbb{R}^{s}, \quad s = C\cdot p^2,
\]
where $P$ denotes the number of patches per field. When $r = p$, the patches are non-overlapping; when $r < p$, they overlap. The patch size $p$ is treated as a hyperparameter, and its effect on reconstruction quality and spectral efficiency is studied in Appendix~\ref{app:Other_ablations}. For single-channel datasets $C = 1.$

We construct a global POD basis by pooling all training patches and computing the economy SVD of the centered patch matrix. The top-$K$ singular vectors form the orthonormal basis $\Phi \in \mathbb{R}^{s \times K}$ with singular values $\sigma_1 \geq \cdots \geq \sigma_{K} \geq 0$.
We select the truncation level $K$ to satisfy the energy criterion
\begin{equation}
    \frac{\sum_{k=1}^{K} \sigma_k^2}{\sum_{k=1}^{s} \sigma_k^2} \geq \eta,
    \qquad \eta = 0.99.
    \label{eq:energy}
\end{equation}
This choice is supported by Proposition~\ref{prop:patch_error}, which shows that the expected reconstruction error is bounded by $(1 - \eta)$ of the total patchwise variance. The shared patch basis keeps token dimensions fixed and improves statistical efficiency by pooling local patches.

\paragraph{Latent encoding.}
Let $\bar{\mathbf{u}}$ denote the global patch mean. Each patch is encoded as
\[
    \mathbf{a}_{i,p} = \Phi^\top
    (\mathbf{u}_{i,p} - \bar{\mathbf{u}}) \in \mathbb{R}^K.
\]
To normalize variance across modes, we standardize the coefficients per mode as
\[
    \tilde{\mathbf{a}}_{i,p} = \Lambda^{-1/2} \mathbf{a}_{i,p},
    \qquad \Lambda = \mathrm{diag}(\sigma_1^2, \dots, \sigma_K^2).
\]
The POD encoder is fixed and parameter-free. In particular, no gradients pass through it during training.


\subsection{Latent Diffusion over Patch Tokens}
\label{sec:diffusion}

\paragraph{Token sequence.}
We represent each field $\mathbf{u}_i$ as a sequence of $P$ latent tokens
\[
    \tilde{\mathbf{A}}_i =
    [\tilde{\mathbf{a}}_{i,1}, \dots, \tilde{\mathbf{a}}_{i,P}]
    \in \mathbb{R}^{P \times K}.
\]
The total latent dimensionality $P \times K$ is substantially smaller than the pixel-space dimensionality $H \times W$, enabling efficient diffusion over the full spatial extent of the field.

\paragraph{Forward process.}
We define the forward diffusion process as
\begin{equation}
    q(\tilde{\mathbf{A}}_t \mid \tilde{\mathbf{A}}_0) =
    \mathcal{N}\!\left(\sqrt{\bar{\alpha}_t}\,\tilde{\mathbf{A}}_0,\;
    (1 - \bar{\alpha}_t)\mathbf{I}\right),
    \label{eq:forward}
\end{equation}
where $\bar{\alpha}_t = \prod_{j=1}^{t} \alpha_j$ follows a cosine noise schedule \cite{nichol2021improved} over $T = 1{,}000$ steps \cite{ho2020denoising}. 

\paragraph{Training objective.}
We train a denoising network ${\epsilon}_\theta$ to predict the injected noise:
\begin{equation}
    \mathcal{L}(\theta) = \mathbb{E}_{\tilde{\mathbf{A}}_0,\, t,\,
     {\epsilon}}\!\left[
    \|{\epsilon} -
     {\epsilon}_\theta(\tilde{\mathbf{A}}_t, \mathbf{C}, t)\|_2^2
    \right],
    \label{eq:loss}
\end{equation}
where $\mathbf{C}$ denotes the conditioning signal derived from the low-resolution input (Section~\ref{sec:conditioning}), and $t \sim \mathrm{Uniform}\{1, \dots, T\}$.
\subsection{Conditioning on Low-Resolution Input}
\label{sec:conditioning}
For super-resolution, we upsample the low-resolution input $\mathbf{x}^{\mathrm{LR}}$ via bicubic interpolation, $\mathbf{x}^{\mathrm{up}} = \mathcal{U}(\mathbf{x}^{\mathrm{LR}}) \in \mathbb{R}^{H \times W}$, and encode it using the same patch-POD pipeline as the high-resolution fields, yielding a conditioning token sequence $\mathbf{C} \in \mathbb{R}^{P \times K}$. Encoding both HR and LR fields in the same latent space enables the denoiser to learn a structured residual in coefficient space. Each token $\mathbf{c}_p$ represents the POD coefficients of the upsampled LR patch at position $p$, providing spatially aligned conditioning.

We apply token-wise additive conditioning by projecting noisy HR tokens and LR tokens to $d_{\mathrm{model}}$ and combining them as
\begin{equation}
    \mathbf{h}_p = \mathbf{W}_{\mathrm{in}}\,\tilde{\mathbf{a}}_{t,p}
    + \mathbf{W}_{\mathrm{cond}}\,\mathbf{c}_p,
    \label{eq:conditioning}
\end{equation}
yielding fused tokens $\mathbf{H} \in \mathbb{R}^{P \times d_{\mathrm{model}}}$. Here $\mathbf{W}_{\mathrm{in}}$ and $\mathbf{W}_{\mathrm{cond}}$ are learnable projection matrices in $\mathbb{R}^{d_{\mathrm{model}}\times K}$. Additive conditioning ensures token alignment and preserves locality. 
\subsection{Vision Transformer Denoising Architecture}
\label{sec:vit}

The denoising network is a Vision Transformer operating on the token sequence $\mathbf{H} \in \mathbb{R}^{P \times d_{\mathrm{model}}}$. Tokens are augmented with 2D positional embeddings and processed by $L$ transformer blocks with multi-head self-attention, MLP layers, and timestep-conditioned Adaptive LayerNorm~\citep{peebles2023dit}. A linear head predicts noise $\hat{\epsilon} \in \mathbb{R}^{P \times K}$.
\subsection{Reconstruction and Uncertainty Quantification}
\label{sec:reconstruction}
At inference, latent tokens are sampled using DDIM~\citep{song2020denoising} with $S=100$ steps and de-normalised to recover POD coefficients,
\[
    \hat{\mathbf{a}}_p = \Lambda^{1/2}\hat{\tilde{\mathbf{a}}}_p.
\]
Each patch is then reconstructed by the linear decoder
\begin{equation}
    \hat{\mathbf{u}}_p = \bar{\mathbf{u}} + \Phi\hat{\mathbf{a}}_p.
    \label{eq:decoder}
\end{equation}
The full field is assembled using a fixed linear stitching operator $\mathcal{S}$. We generate $M$ independent latent samples to estimate coefficient-level covariance, and then propagate this covariance through the fixed POD decoder to obtain pixel-space predictive variance. Thus, sampling is performed in the low-dimensional POD space, while spatial uncertainty is obtained through the known decoder geometry.
\begin{figure}
  \centering
  \includegraphics[width=\columnwidth]{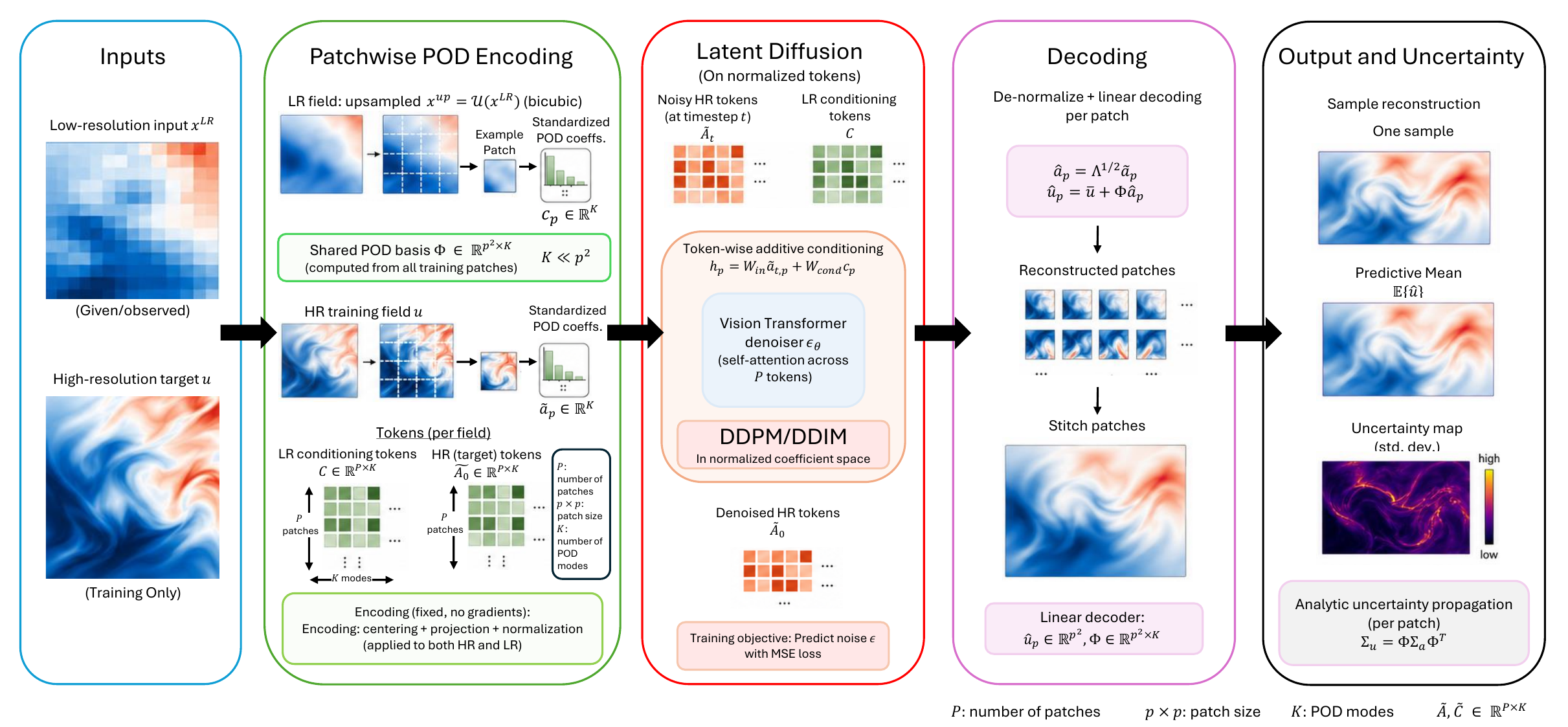}
  \caption{Overview of the proposed Patch-PODiff-ViT pipeline for conditional super-resolution in a structured latent space.}
  \label{fig:pipeline}
\end{figure}
\subsection{Theoretical Guarantees}
\label{sec:theory}
The linear POD encoder-decoder yields two key results.

\subsubsection{Proposition 1: Reconstruction Error Bound}
\begin{proposition}[Patchwise POD Reconstruction Error Bound]
\label{prop:patch_error}
Let random field $\mathbf{u} \in \mathbb{R}^d$ be partitioned into $P$ non-overlapping patches $\mathbf{u}_p \in \mathbb{R}^s$. Let $\bar{\mathbf{u}}$ be the global patch mean and $\Phi \in \mathbb{R}^{s \times K}$ the POD basis retaining energy fraction $\eta$. Define $\hat{\mathbf{u}}_p = \bar{\mathbf{u}} + \Phi \Phi^\top(\mathbf{u}_p - \bar{\mathbf{u}})$ and let $\hat{\mathbf{u}}$ be the reconstructed field. Then
\[
\mathbb{E}\!\left[\|\mathbf{u} - \hat{\mathbf{u}}\|^2\right]
\leq
(1 - \eta) \sum_{p=1}^{P}
\mathbb{E}\!\left[\|\mathbf{u}_p - \bar{\mathbf{u}}\|^2\right].
\]
\end{proposition}

\noindent The result follows from the Eckart--Young theorem; see Appendix~\ref{app:proof1}. We use $\eta = 0.99$ and validate this bound in Section~\ref{sec:bound_validation}.

\subsubsection{Proposition 2: Analytic Uncertainty Propagation}
\begin{proposition}[Analytic Uncertainty Propagation under Patchwise POD]
\label{prop:patch_uq}
Let $\hat{\mathbf{u}}_p = \bar{\mathbf{u}} + \Phi\mathbf{a}_p$ with $\Sigma_{a_p} = \mathrm{Cov}(\mathbf{a}_p)$. Then:
\begin{enumerate}
\item[\textbf{(i)}] $\Sigma_{u_p} = \Phi\,\Sigma_{a_p}\,\Phi^\top$,
\item[\textbf{(ii)}] for non-overlapping patches, assuming $\mathrm{Cov}(\mathbf{a}_p, \mathbf{a}_q)=0$ for $p \neq q$, $\Sigma_u = \mathrm{blkdiag}(\Phi\Sigma_{a_1}\Phi^\top,\dots,\Phi\Sigma_{a_P}\Phi^\top)$,
\item[\textbf{(iii)}] for overlapping patches, $\Sigma_u = \mathcal{S}\tilde{\Sigma}\mathcal{S}^\top$, where $\tilde{\Sigma}$ stacks the patch covariances. 
\end{enumerate}
\end{proposition}

\noindent The result follows from linearity; see Appendix~\ref{app:proof2}. In practice, we estimate $\hat{\Sigma}_{a_p}$ from $M$ generated latent samples and obtain pixel-level predictive variance as
$\mathrm{diag}(\Phi\hat{\Sigma}_{a_p}\Phi^\top)$ at cost $\mathcal{O}(PMK^2 + PsK)$,
without explicitly forming or storing a full $HW \times HW$ pixel-space covariance matrix.

\subsection{Computational Efficiency}
\label{sec:efficiency}
Patch-PODiff-ViT is efficient due to (i) low-dimensional tokens ($K \ll s$), (ii) resolution-independent tokenisation, and (iii) parameter-free POD encoding. See Section~\ref{sec:efficiency_results} for quantitative comparisons.
\subsection{Limitations}
\label{sec:limitations}
Patch-PODiff-ViT is most effective when spatial fields exhibit low-rank local structure. For highly turbulent or discontinuous fields with slow singular value decay, more POD modes are required, increasing latent dimensionality $P \times K$ and narrowing the efficiency gap relative to pixel-space methods. We observe this in Appendix~\ref{app:advection}, where at $\mathrm{Pe} = O(10^6)$, up to $K=150$ modes are needed to retain $99\%$ variance, compared to $K=2$--$26$ in our main datasets. The POD basis is fixed after construction, so distribution shifts may require recomputation, with adaptive or online extensions left for future work. Proposition~\ref{prop:patch_uq} uses a block-diagonal covariance approximation that neglects cross-patch terms, trading modeling fidelity for tractability. This affects only the covariance map, since the ViT still captures cross-patch dependencies during sampling, and overlapping aggregation via $\mathcal{S}$ can partially mitigate boundary effects. Discarded-mode uncertainty is not modeled explicitly, but remains small under $\eta=0.99$ by Proposition~\ref{prop:patch_error}.
\section{Experimental Setup}
\label{sec:experiments}
\subsection{Datasets}
\label{sec:datasets}
We evaluate on three domains: geophysical fields, medical images, and natural images. 
For SST, we use daily Western Australian coast fields at $640\times480$ from ROMS~\citep{shchepetkin2005regional}, conditioned on ACCESS-S2 inputs at $53\times31$~\citep{wedd2022access} bicubically interpolated to the target grid. We train on 1998--2009 ($\approx$4{,}000 fields), validate on 2010, and test on 2011, which includes the documented West Australian marine heatwave. Land pixels are masked and metrics are computed over ocean pixels only. 
For medical imaging, we use NIH ChestX-ray14~\citep{wang2017chestxray}, resizing frontal radiographs to $256\times256$ and following the official train/test split. 
For natural images, we use FFHQ~\citep{karras2019style} at $256\times256$ with an 80/20 random split using seed 0. 
For X-ray and FFHQ, low-resolution inputs are produced by $4\times$ bicubic downsampling to $64\times64$ followed by upsampling to $256\times256$ before encoding.
\subsection{Baselines and Implementation Details}
\label{sec:baselines}
We compare against five baselines covering deterministic, latent diffusion, pixel-space diffusion, transformer diffusion, and POD-based variants.

\textbf{U-Net:} A deterministic convolutional U-Net~\citep{ronneberger2015u} with a symmetric encoder--decoder, four resolution levels, base width $C=128$, channel widths $(C,2C,4C,8C)$, skip connections, and an $\ell_2$ loss. The same architecture is used across all datasets.

\textbf{VAE-LDM:} A latent diffusion model following Rombach et al.~\citep{rombach2022high}, with a convolutional VAE encoder of base width $C=128$ that compresses $256\times256$ inputs to $64\times64\times4$ latent maps.
The denoiser uses the same ViT configuration as Patch-PODiff-ViT, providing a controlled learned-latent baseline, although the latent dimensionalities are not identical. The VAE is trained on the same training split and selected using validation reconstruction quality before training the latent diffusion model.

\textbf{DiT:} A Diffusion Transformer~\citep{peebles2023dit} operating directly on $16 \times 16$ pixel patches without POD compression. It uses the same ViT configuration as Patch-PODiff-ViT, namely $d_{\mathrm{model}}=512$, $L=12$, and $H=8$, controlling for denoiser capacity; both methods use the same patch sequence length, while Patch-PODiff-ViT compresses each token to $K$ POD modes.

\textbf{PixelDiff:} A pixel-space diffusion model using the same U-Net backbone as the deterministic baseline, trained with a DDPM objective and conditioned on the bicubically upsampled LR input via channel concatenation, following the residual diffusion paradigm of CorrDiff~\cite{mardani2025residual}.

\textbf{Fullfield-PODiff:} A non-patchwise POD baseline using a single full-field POD basis, with $K$ selected by the same $\eta=0.99$ criterion, isolating the contribution of patchwise structure.
%
%

All diffusion models use a cosine noise schedule with $T=1{,}000$ training timesteps, $S=100$ inference steps, and $M=100$ ensemble samples for uncertainty. Models are trained with AdamW at learning rate $10^{-4}$ using identical preprocessing and hardware. Hyperparameters are selected on validation splits and fixed for test evaluation. Standard deviations are reported in Appendix~\ref{app:extended_results}.

\textbf{Patch-PODiff-ViT Implementation Details:} The POD basis is computed from training patches at $\eta = 0.99$. The ViT denoiser uses $d_{\mathrm{model}} = 512$, $L = 12$, $H = 8$ heads, and MLP ratio $4.0$ (${\approx}70$M parameters), and is trained for 1M steps with batch size 64, AdamW (lr$=10^{-4}$, weight decay $10^{-4}$), gradient clipping 1.0, and EMA 0.9995. We use a cosine schedule ($T=1{,}000$) and DDIM sampling  ($S=100$) at inference. Uncertainty estimates use $M=100$ ensemble members unless stated otherwise. The default patch size and resulting $K$ values are reported alongside the main results. Patch size sensitivity is in  Appendix~\ref{app:Other_ablations}. All models trained on AMD Instinct MI250X GPUs.

%
\subsection{Evaluation Metrics}
\label{sec:metrics}
We evaluate reconstruction using RMSE, PSNR, SSIM~\citep{wang2004image}, LPIPS~\citep{zhang2018unreasonable}, and FID~\citep{heusel2017gans}. Uncertainty is assessed with CRPS~\citep{gneiting2007strictly}, MACE, and empirical coverage at nominal levels $\{50,60,70,80,90,95,99\}\%$ using $M=100$ samples, with SST UQ computed over all 365 test days in 2011. Efficiency is measured by parameter count, peak GPU memory, training time, per-sample inference time, and $M=100$ ensemble inference time. All models use identical preprocessing: SST is min--max normalized over training ocean pixels with land masked out, while Chest X-ray and FFHQ are scaled to $[0,1]$, with LPIPS and FID using the required pretrained-network input ranges.
%
\section{Results}
\label{sec:results}
We evaluate Patch-PODiff-ViT using reconstruction accuracy, theoretical bound validation, uncertainty quantification, computational efficiency, and downstream SST ensemble forecasting. The experiments test three claims: patchwise POD provides a compact but spatially structured latent space; the fixed linear decoder enables accurate propagation of latent uncertainty to pixel-space variance; and the efficiency advantage is largest when local patch spectra decay rapidly. Unless stated otherwise, all results use $16 \times 16$ patches, $\eta = 0.99$ energy threshold, and $M=100$ ensemble members. Ensemble size $M$ sensitivity is examined in Appendix~\ref{app:ensemble_size}.
\subsection{Reconstruction Accuracy}
\label{sec:reconstruction_accuracy}
\begin{figure}[t]
  \centering
  \includegraphics[width=\columnwidth]{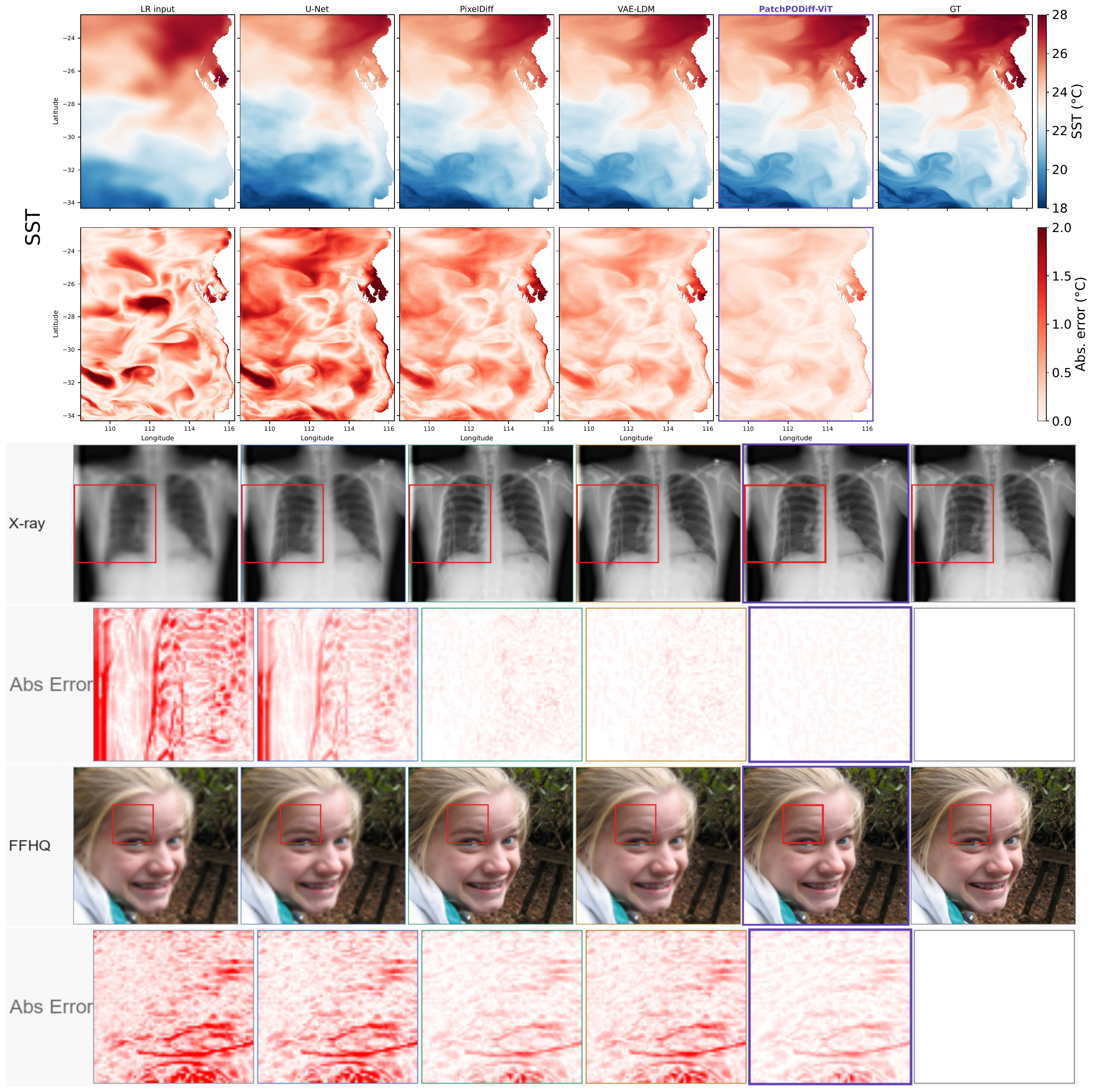}
  \caption{Qualitative comparison across three domains. Columns: LR input, U-Net, PixelDiff, VAE-LDM, Patch-PODiff-ViT (ours), GT. SST (rows 1--2): full reconstruction and absolute error. Chest X-ray (rows 3--4): full reconstruction and zoomed absolute error. FFHQ (rows 5--6): full reconstruction and zoomed absolute error. Patch-PODiff-ViT consistently produces lower error across all domains. Full comparisons with zoomed regions are in Appendix~\ref{app:qualitative}.}
  \label{fig:qualitative_comparison}
\end{figure}
\begin{table*}[t]
\centering
\caption{Reconstruction quality across three datasets. Lower is better for RMSE, LPIPS, FID; higher for PSNR, SSIM. Standard deviations are provided in Appendix~\ref{app:extended_results} in Table \ref{tab:model_comparison_full}. Results are averaged over 3 seed runs.}
\label{tab:model_comparison}
\scriptsize
\setlength{\tabcolsep}{3pt}
\begin{tabular}{l c c c c c c}
\toprule
Metric & \textbf{Patch-PODiff-ViT} & VAE-LDM & DiT & PixelDiff & U-Net & Fullfield-PODiff \\
\midrule
\multicolumn{7}{c}{\textit{SST (Sea Surface Temperature)}} \\
\midrule
RMSE $\downarrow$  & $\mathbf{0.0030}$ & $0.0041$ & $0.0040$ & $0.0049$ & $0.0093$ & $0.0049$ \\
PSNR $\uparrow$    & $\mathbf{50.43}$  & $47.73$  & $47.95$  & $46.23$  & $40.63$  & $46.22$  \\
SSIM $\uparrow$    & $\mathbf{0.9888}$ & $0.9799$ & $0.9707$ & $0.9713$ & $0.9589$ & $0.9690$ \\
LPIPS $\downarrow$ & $\mathbf{0.0131}$ & $0.0201$ & $0.0169$ & $0.0234$ & $0.0295$ & $0.0267$ \\
FID $\downarrow$   & $\mathbf{3.986}$  & $5.90$   & $5.01$   & $8.46$   & $11.96$  & $8.49$   \\
\midrule
\multicolumn{7}{c}{\textit{Chest X-ray}} \\
\midrule
RMSE $\downarrow$  & $\mathbf{0.0065}$ & $0.0087$ & $0.0081$ & $0.0098$ & $0.0133$ & $0.0092$ \\
PSNR $\uparrow$    & $\mathbf{42.98}$  & $39.11$  & $41.21$  & $37.64$  & $34.52$  & $35.75$  \\
SSIM $\uparrow$    & $\mathbf{0.9885}$ & $0.9790$ & $0.9632$ & $0.9743$ & $0.9423$ & $0.9699$ \\
LPIPS $\downarrow$ & $\mathbf{0.0201}$ & $0.0298$ & $0.0281$ & $0.0310$ & $0.0422$ & $0.0307$ \\
FID $\downarrow$   & $\mathbf{6.0152}$ & $8.17$   & $7.01$   & $11.44$  & $13.65$  & $9.26$   \\
\midrule
\multicolumn{7}{c}{\textit{FFHQ (Face Images)}} \\
\midrule
RMSE $\downarrow$  & $\mathbf{0.0109}$ & $0.0178$ & $0.0121$ & $0.0147$ & $0.0361$ & $0.0154$ \\
PSNR $\uparrow$    & $\mathbf{39.15}$  & $36.44$  & $38.04$  & $37.87$  & $32.54$  & $37.15$  \\
SSIM $\uparrow$    & $\mathbf{0.9522}$ & $0.9321$ & $0.9432$ & $0.9411$ & $0.9255$ & $0.9358$ \\
LPIPS $\downarrow$ & $\mathbf{0.0300}$ & $0.0561$ & $0.0377$ & $0.0427$ & $0.0584$ & $0.0471$ \\
FID $\downarrow$   & $\mathbf{9.168}$  & $11.98$  & $10.04$  & $10.11$  & $16.18$  & $12.08$  \\
\bottomrule
\end{tabular}
\end{table*}

Table~\ref{tab:model_comparison} shows that Patch-PODiff-ViT achieves consistently strong reconstruction performance across all datasets, with the best result on most reported metrics.
Among diffusion baselines, DiT is the strongest competitor, suggesting that the gains arise not only from the transformer denoiser but also from the structured POD latent space. Appendix \ref{app:compressionLevel} confirms that variance-ordered compression, not linear structure alone, drives these gains. 
Additionally, since VAE-LDM uses the same denoiser configuration, these gains suggest that the structured POD latent representation is advantageous in these settings, though the comparison does not fully equalize latent dimensionality.
Figure~\ref{fig:qualitative_comparison} shows qualitative results across all three domains. Patch-PODiff-ViT consistently preserves large-scale structure and fine-scale detail, including fronts in SST, anatomical boundaries in Chest X-ray, and texture in FFHQ, while competing methods exhibit smoothing or local distortions near sharp gradients. This trend is reflected quantitatively in Table~\ref{tab:model_comparison} and further supported by extended comparisons in Appendix~\ref{app:qualitative}.

Fullfield-PODiff further ablates patchwise structure: despite using a POD latent space, its FID is consistently worse than Patch-PODiff-ViT. This suggests that local patchwise POD, rather than POD compression alone, drives the improvement. 
The U-Net underperforms all diffusion baselines across perceptual metrics, consistent with the known limitations of deterministic regression for high-frequency detail recovery.
\subsection{Validation of the POD Reconstruction Bound}
\label{sec:bound_validation}
We empirically validate the retained-energy reconstruction bound in Appendix~\ref{app:pod_analysis}. 
Across all datasets, mean empirical POD errors remain below the expected bound; the above-97\% sample-wise satisfaction rates are only empirical diagnostics. This confirms that the retained-energy criterion in Eq. (\ref{eq:energy}) provides reliable control of patchwise reconstruction error in the settings considered.
\subsection{Uncertainty Quantification}
\label{sec:results_uq}
Uncertainty quantification is a central component of the proposed method. We evaluate it along three axes: (i) validity of analytic propagation (Proposition~\ref{prop:patch_uq}), (ii) calibration of predictive intervals, and (iii) sharpness of the predictive distribution.

\textbf{Analytic vs. empirical uncertainty:}
\begin{figure}[htbp]
  \centering
  \includegraphics[width=0.8\columnwidth]{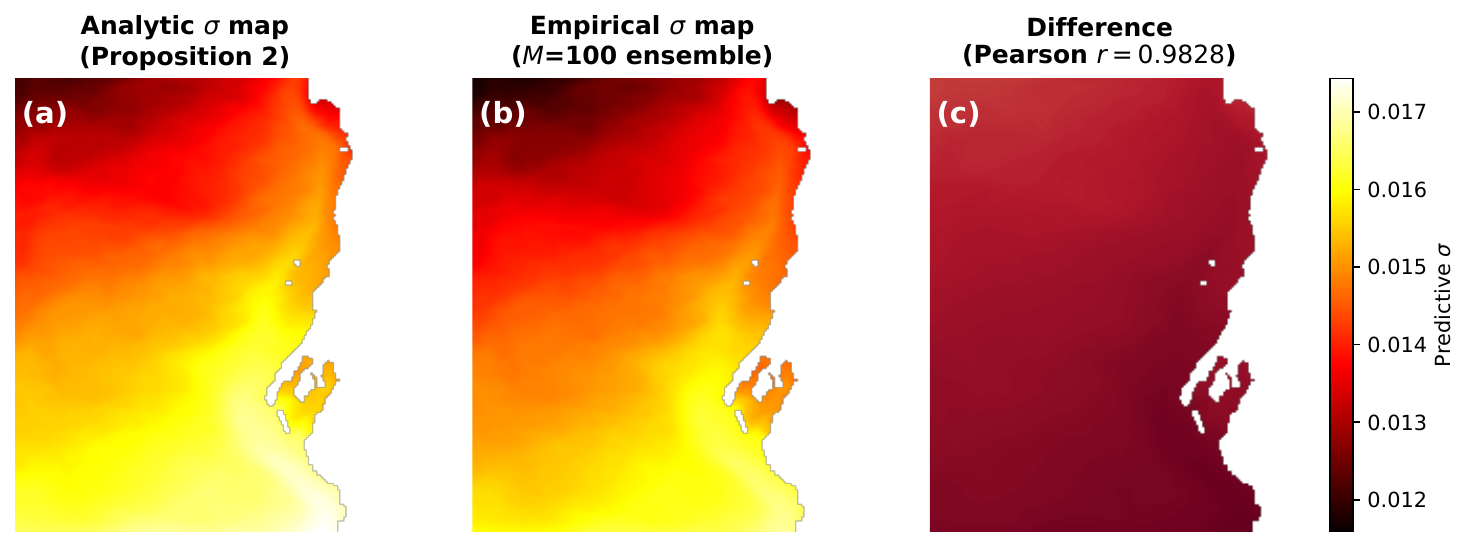}
  \caption{Analytic versus empirical uncertainty on SST. Left: analytic standard deviation computed via Proposition~\ref{prop:patch_uq}. Middle: empirical standard deviation from an ensemble of $M=100$ samples. Right: difference map, with Pearson correlation $r=0.983$. The close agreement indicates that the analytic formulation captures both the spatial structure and magnitude of uncertainty without explicitly forming full pixel-space covariance estimates.}
  \label{fig:uq_maps_sst}
\end{figure}
Figure~\ref{fig:uq_maps_sst} evaluates the covariance propagation step in Proposition~\ref{prop:patch_uq} by comparing pixel-wise standard deviation maps obtained from $\mathrm{diag}(\Phi\hat{\Sigma}_{a_p}\Phi^\top)$ with empirical standard deviation maps computed from $M=100$ decoded samples. The two agree closely on SST ($r=0.9828$), with similar correlations on Chest X-ray ($r=0.986$) and FFHQ ($r=0.953$). Additional maps are provided in Appendix~\ref{app:uq_maps}. Including covariance estimation, analytic propagation costs
$\mathcal{O}(PMK^2 + PsK)$ and avoids explicitly forming full pixel-space covariance estimates. Appendix~\ref{app:cross_patch_covariance} further supports the block-diagonal covariance approximation: off-diagonal cross-patch covariance is small for SST and Chest X-ray and moderate for FFHQ, consistent with the mild patch-boundary effects observed on natural images. Spatially, uncertainty concentrates near SST coastlines and thermal fronts, X-ray lung boundaries and rib edges, and FFHQ features such as eyes and hair.

\textbf{Calibration and Sharpness:}
Figure~\ref{fig:reliability} shows that Patch-PODiff-ViT closely follows the ideal calibration line across datasets and nominal levels. It achieves the lowest MACE, with $0.008$ on SST, $0.0046$ on Chest X-ray, and $0.0084$ on FFHQ, while DiT and VAE-LDM show larger deviations. For sharpness, Patch-PODiff-ViT obtains the lowest CRPS on SST and Chest X-ray. On FFHQ, VAE-LDM attains lower CRPS but higher MACE, indicating a sharper but less calibrated predictive distribution. Full coverage and CRPS values are provided in Appendix~\ref{app:coverage}.
\begin{figure}[t]
  \centering
  \includegraphics[width=0.85\columnwidth]{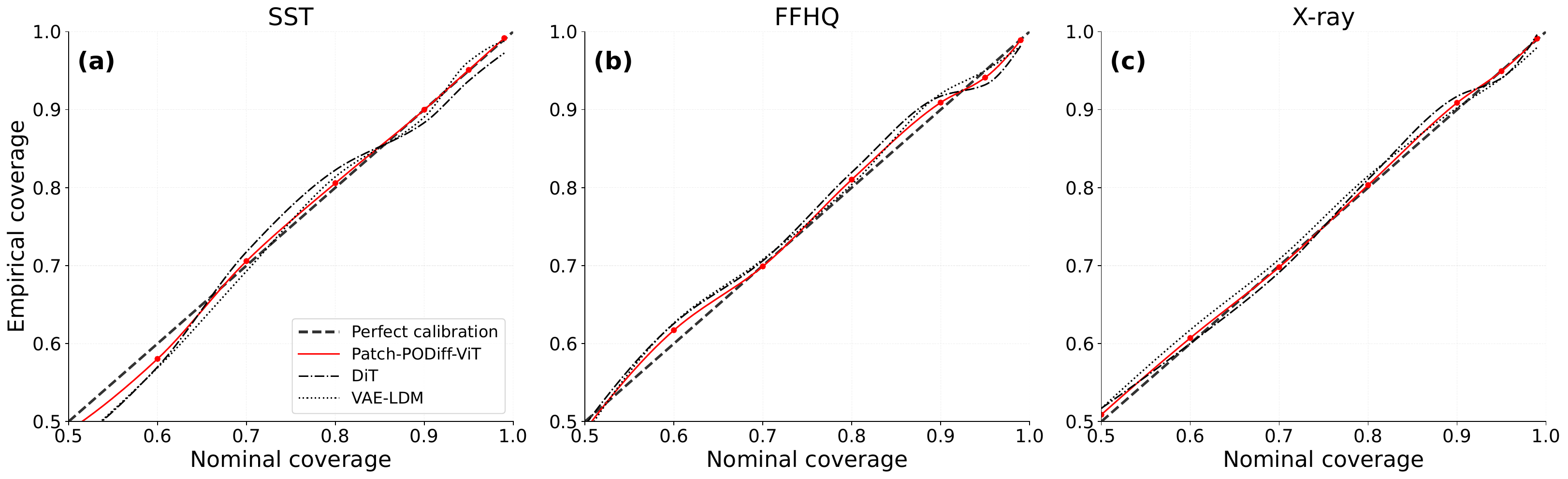}
  \caption{Reliability diagrams comparing empirical versus nominal coverage for SST, FFHQ, and Chest X-ray. The dashed line denotes perfect calibration. Patch-PODiff-ViT remains close to the ideal line across datasets, indicating well-calibrated predictive intervals.}
  \label{fig:reliability}
\end{figure}
Together, these results show that Patch-PODiff-ViT produces uncertainty estimates that are both accurate and well calibrated. 
The POD-based analytic propagation enables efficient uncertainty estimation, while maintaining competitive calibration and sharpness across diverse domains.
\subsection{Downstream Application: Spatial Uncertainty for Ensemble Forecasting}
\label{sec:downstream_sst}
\begin{figure}[htbp]
  \centering
  \includegraphics[width=0.85\columnwidth]{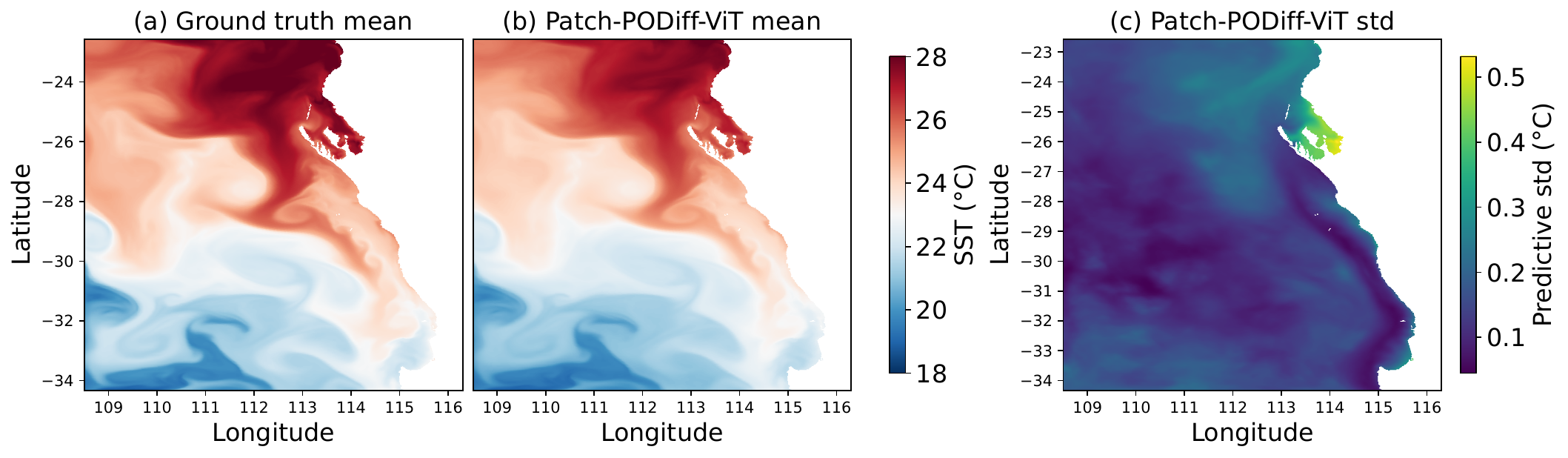}
  \caption{SST ensemble downstream application. (a) Ground truth SST field, (b) Patch-PODiff-ViT ensemble mean, and (c) predictive standard deviation in physical units ($^\circ$C). The uncertainty map highlights regions of higher predictive uncertainty along coastal areas and strong thermal gradients.}
  \label{fig:sst_downstream}
\end{figure}
Figure~\ref{fig:sst_downstream} illustrates Patch-PODiff-ViT on a representative day from the 2011 West Australian marine heatwave. The ensemble mean recovers the large-scale warm water mass, sharp meridional gradient, and coastal fine-scale structure, while the predictive standard deviation highlights higher uncertainty near coastlines and strong thermal gradients. From an application perspective, the model provides both a high-resolution estimate and a spatial confidence map, helping identify regions that require greater caution.
\subsection{Computational Efficiency}
\label{sec:efficiency_results}
\begin{table}[htbp]
\centering
\caption{Computational cost comparison across diffusion models.}
\label{tab:efficiency}
\scriptsize
\setlength{\tabcolsep}{4pt}
\begin{tabular}{lccccc}
\toprule
Model & Params (M) & Train (hrs) & GPU Mem (GB) & Inference (s) & $M{=}100$ (s) \\
\midrule
Patch-PODiff-ViT & 70  & 8.7  & 8.6  & 0.109 & 11.036 \\
VAE-LDM          & 220 & 30.6 & 16.1 & 0.255 & 26.038 \\
DiT              & 75  & 10.2 & 9.2  & 0.112 & 11.855 \\
PixelDiff        & 160 & 26.8 & 12.3 & 1.800 & 189.530 \\
\bottomrule
\end{tabular}
\end{table}
Table~\ref{tab:efficiency} compares computational cost. Patch-PODiff-ViT uses 70M parameters and trains in 8.7 h, compared with 220M and 30.6 h for VAE-LDM. It also reduces peak memory from 16.1 GB to 8.6 GB and generates an $M=100$ ensemble in 11.036 s, compared with 26.038 s for VAE-LDM and 189.530 s for PixelDiff. DiT has similar sampling time due to the matched transformer backbone, but Patch-PODiff-ViT achieves better reconstruction and calibration through POD-compressed tokens. 

Additional ablations on patch size and denoiser architecture are provided in Appendix~\ref{app:Other_ablations}, showing that $16\times 16$ performs best among the tested patch sizes and that ViT self-attention is critical, as replacing the ViT with a per-token MLP substantially degrades performance (FID 3.986 to 12.59).

\bibliographystyle{unsrt}
\bibliography{references}

@article{ho2020denoising,
  title={Denoising diffusion probabilistic models},
  author={Ho, Jonathan and Jain, Ajay and Abbeel, Pieter},
  journal={Advances in neural information processing systems},
  volume={33},
  pages={6840--6851},
  year={2020}
}

@inproceedings{
song2020score,
title={Score-Based Generative Modeling through Stochastic Differential Equations},
author={Yang Song and Jascha Sohl-Dickstein and Diederik P Kingma and Abhishek Kumar and Stefano Ermon and Ben Poole},
booktitle={International Conference on Learning Representations},
year={2021},
url={https://openreview.net/forum?id=PxTIG12RRHS}
}

@inproceedings{nichol2021improved,
  title={Improved denoising diffusion probabilistic models},
  author={Nichol, Alexander Quinn and Dhariwal, Prafulla},
  booktitle={International conference on machine learning},
  pages={8162--8171},
  year={2021},
  organization={PMLR}
}

@inproceedings{
song2020denoising,
title={Denoising Diffusion Implicit Models},
author={Jiaming Song and Chenlin Meng and Stefano Ermon},
booktitle={International Conference on Learning Representations},
year={2021},
url={https://openreview.net/forum?id=St1giarCHLP}
}

@article{vahdat2021score,
  title={Score-based generative modeling in latent space},
  author={Vahdat, Arash and Kreis, Karsten and Kautz, Jan},
  journal={Advances in neural information processing systems},
  volume={34},
  pages={11287--11302},
  year={2021}
}

@article{price2023gencast,
  title={Gencast: Diffusion-based ensemble forecasting for medium-range weather},
  author={Price, Ilan and Sanchez-Gonzalez, Alvaro and Alet, Ferran and Andersson, Tom R and El-Kadi, Andrew and Masters, Dominic and Ewalds, Timo and Stott, Jacklynn and Mohamed, Shakir and Battaglia, Peter and others},
  journal={arXiv preprint arXiv:2312.15796},
  year={2023}
}

@article{leinonen2020stochastic,
  title={Stochastic super-resolution for downscaling time-evolving atmospheric fields with a generative adversarial network},
  author={Leinonen, Jussi and Nerini, Daniele and Berne, Alexis},
  journal={IEEE Transactions on Geoscience and Remote Sensing},
  volume={59},
  number={9},
  pages={7211--7223},
  year={2020},
  publisher={IEEE}
}

@article{stengel2020adversarial,
  title={Adversarial super-resolution of climatological wind and solar data},
  author={Stengel, Karen and Glaws, Andrew and Hettinger, Dylan and King, Ryan N},
  journal={Proceedings of the National Academy of Sciences},
  volume={117},
  number={29},
  pages={16805--16815},
  year={2020},
  publisher={National Academy of Sciences}
}

@article{benner2015survey,
  title={A survey of projection-based model reduction methods for parametric dynamical systems},
  author={Benner, Peter and Gugercin, Serkan and Willcox, Karen},
  journal={{SIAM} Review},
  volume={57},
  number={4},
  pages={483--531},
  year={2015},
  publisher={SIAM}
}

@article{gneiting2007strictly,
  title={Strictly proper scoring rules, prediction, and estimation},
  author={Gneiting, Tilmann and Raftery, Adrian E},
  journal={Journal of the American Statistical Association},
  volume={102},
  number={477},
  pages={359--378},
  year={2007},
  publisher={Taylor \& Francis}
}

@article{watt2024generative,
  title={Generative diffusion-based downscaling for climate},
  author={Watt, Robbie A and Mansfield, Laura A},
  journal={arXiv preprint arXiv:2404.17752},
  year={2024}
}

@article{coscia2024generative,
  title={Generative adversarial reduced order modelling},
  author={Coscia, Dario and Demo, Nicola and Rozza, Gianluigi},
  journal={Scientific Reports},
  volume={14},
  number={1},
  pages={3826},
  year={2024},
  publisher={Nature Publishing Group UK London}
}

@inproceedings{rombach2022high,
  title={High-resolution image synthesis with latent diffusion models},
  author={Rombach, Robin and Blattmann, Andreas and Lorenz, Dominik and Esser, Patrick and Ommer, Bj{\"o}rn},
  booktitle={Proceedings of the IEEE/CVF conference on computer vision and pattern recognition},
  pages={10684--10695},
  year={2022}
}

@article{saharia2022image,
  title={Image super-resolution via iterative refinement},
  author={Saharia, Chitwan and Ho, Jonathan and Chan, William and Salimans, Tim and Fleet, David J and Norouzi, Mohammad},
  journal={IEEE transactions on pattern analysis and machine intelligence},
  volume={45},
  number={4},
  pages={4713--4726},
  year={2022},
  publisher={IEEE}
}

@article{moser2024diffusion,
  title={Diffusion models, image super-resolution, and everything: A survey},
  author={Moser, Brian B and Shanbhag, Arundhati S and Raue, Federico and Frolov, Stanislav and Palacio, Sebastian and Dengel, Andreas},
  journal={IEEE Transactions on Neural Networks and Learning Systems},
  year={2024},
  publisher={IEEE}
}

@article{berkooz1993proper,
  title={The proper orthogonal decomposition in the analysis of turbulent flows},
  author={Berkooz, Gal and Holmes, Philip and Lumley, John L},
  journal={Annual review of fluid mechanics},
  volume={25},
  number={1},
  pages={539--575},
  year={1993},
  publisher={Annual Reviews 4139 El Camino Way, PO Box 10139, Palo Alto, CA 94303-0139, USA}
}

@article{sirovich1987turbulence,
  title={Turbulence and the dynamics of coherent structures. {I.} Coherent structures},
  author={Sirovich, Lawrence},
  journal={Quarterly of applied mathematics},
  volume={45},
  number={3},
  pages={561--571},
  year={1987}
}

@inproceedings{ronneberger2015u,
  title={U-net: Convolutional networks for biomedical image segmentation},
  author={Ronneberger, Olaf and Fischer, Philipp and Brox, Thomas},
  booktitle={International Conference on Medical image computing and computer-assisted intervention},
  pages={234--241},
  year={2015},
  organization={Springer}
}

@article{leinonen2023latent,
  title={Latent diffusion models for generative precipitation nowcasting with accurate uncertainty quantification},
  author={Leinonen, Jussi and Hamann, Ulrich and Nerini, Daniele and Germann, Urs and Franch, Gabriele},
  journal={arXiv preprint arXiv:2304.12891},
  year={2023}
}

@article{shchepetkin2005regional,
  title={The regional oceanic modeling system {(ROMS)}: a split-explicit, free-surface, topography-following-coordinate oceanic model},
  author={A.F. Shchepetkin and J.C. McWilliams},
  journal={Ocean Modelling},
  volume={9},
  number={4},
  pages={347--404},
  year={2005},
  publisher={Elsevier}
}

@article{wedd2022access,
  title={{ACCESS-S2:} the upgraded Bureau of Meteorology multi-week to seasonal prediction system},
  author={R. Wedd and O. Alves and C. {de Burgh-Day} and C. Down and M. Griffiths and H.H. Hendon and D. Hudson and S. Li and E.P. Lim and A.G. Marshall and others},
  journal={Journal of Southern Hemisphere Earth Systems Science},
  volume={72},
  number={3},
  pages={218--242},
  year={2022},
  publisher={CSIRO Publishing}
}

@article{du2024conditional,
  title={Conditional neural field latent diffusion model for generating spatiotemporal turbulence},
  author={Du, Pan and Parikh, Meet Hemant and Fan, Xiantao and Liu, Xin-Yang and Wang, Jian-Xun},
  journal={Nature Communications},
  volume={15},
  number={1},
  pages={10416},
  year={2024},
  publisher={Nature Publishing Group UK London}
}

@inproceedings{peebles2023dit,
  title={Scalable Diffusion Models with Transformers},
  author={Peebles, William and Xie, Saining},
  booktitle={Proceedings of the IEEE/CVF International Conference on Computer Vision},
  pages={4195--4205},
  year={2023}
}

@inproceedings{wang2017chestxray,
  title={{ChestX-ray8}: Hospital-scale Chest X-ray Database and Benchmarks on Weakly-Supervised Classification and Localization of Common Thorax Diseases},
  author={Wang, Xiaosong and Peng, Yifan and Lu, Le and Lu, Zhiyong and Bagheri, Mohammadhadi and Summers, Ronald M.},
  booktitle={Proceedings of the IEEE Conference on Computer Vision and Pattern Recognition},
  pages={2097--2106},
  year={2017}
}

@article{wang2004image,
  title={Image Quality Assessment: From Error Visibility to Structural Similarity},
  author={Wang, Zhou and Bovik, Alan C. and Sheikh, Hamid R. and Simoncelli, Eero P.},
  journal={IEEE Transactions on Image Processing},
  volume={13},
  number={4},
  pages={600--612},
  year={2004}
}

@inproceedings{zhang2018unreasonable,
  title={The Unreasonable Effectiveness of Deep Features as a Perceptual Metric},
  author={Zhang, Richard and Isola, Phillip and Efros, Alexei A. and Shechtman, Eli and Wang, Oliver},
  booktitle={Proceedings of the IEEE Conference on Computer Vision and Pattern Recognition},
  pages={586--595},
  year={2018}
}

@inproceedings{heusel2017gans,
  title={{GANs} Trained by a Two Time-Scale Update Rule Converge to a Local {Nash} Equilibrium},
  author={Heusel, Martin and Ramsauer, Hubert and Unterthiner, Thomas and Nessler, Bernhard and Hochreiter, Sepp},
  booktitle={Advances in Neural Information Processing Systems},
  volume={30},
  year={2017}
}

@inproceedings{karras2019style,
  title={A Style-Based Generator Architecture for Generative Adversarial Networks},
  author={Karras, Tero and Laine, Samuli and Aila, Timo},
  booktitle={Proceedings of the IEEE/CVF Conference on Computer Vision and Pattern Recognition},
  pages={4401--4410},
  year={2019}
}

@article{bohm2019uncertainty,
  title={Uncertainty quantification with generative models},
  author={B{\"o}hm, Vanessa and Lanusse, Fran{\c{c}}ois and Seljak, Uro{\v{s}}},
  journal={arXiv preprint arXiv:1910.10046},
  year={2019}
}

@article{mardani2025residual,
  title={Residual corrective diffusion modeling for km-scale atmospheric downscaling},
  author={Mardani, Morteza and Brenowitz, Noah and Cohen, Yair and Pathak, Jaideep and Chen, Chieh-Yu and Liu, Cheng-Chin and Vahdat, Arash and Nabian, Mohammad Amin and Ge, Tao and Subramaniam, Akshay and others},
  journal={Communications Earth \& Environment},
  volume={6},
  number={1},
  pages={124},
  year={2025},
  publisher={Nature Publishing Group UK London}
}

@article{jadhav2026podiff,
  title={{PODiff:} Latent Diffusion in Proper Orthogonal Decomposition Space for Scientific Super-Resolution},
  author={Jadhav, Onkar and French, Tim and Rayson, Matthew and Jones, Nicole L},
  journal={arXiv preprint arXiv:2605.03399},
  year={2026}
}

@article{Jadhav2025RCNN,
author = {Onkar Jadhav  and Tim French  and Ivica Janekovic  and Nicole L Jones  and Matthew David Rayson },
title = {Deep Learning-Based Statistical Downscaling of Sea Surface Temperature Using a Residual Corrective Neural Network},
journal = {ESS Open Archive},
volume = {2025},
number = {0813},
pages = {},
year = {2025},
doi = {10.22541/essoar.175510661.11731470/v1},
URL = {https://essopenarchive.org/doi/abs/10.22541/essoar.175510661.11731470/v1},
eprint = {https://essopenarchive.org/doi/pdf/10.22541/essoar.175510661.11731470/v1}
}








\appendix

\section{Proposition 1 Proof}
\label{app:proof1}
\setcounter{proposition}{0}

\begin{proposition}[Patchwise POD Reconstruction Error Bound]
Let random field $\mathbf{u} \in \mathbb{R}^d$ be partitioned into $P$ non-overlapping
patches $\{\mathbf{u}_p\}_{p=1}^{P}$, where each patch
$\mathbf{u}_p \in \mathbb{R}^s$ and $d = Ps$.

Let $\bar{\mathbf{u}} \in \mathbb{R}^s$ be the global patch mean and
$\Phi \in \mathbb{R}^{s \times K}$ the shared POD basis formed by the top-$K$
eigenvectors of the pooled patch covariance, with corresponding eigenvalues
\[
\lambda_1 \ge \lambda_2 \ge \cdots \ge \lambda_s \ge 0.
\]
Assume the retained modes satisfy
\[
\frac{\sum_{k=1}^{K}\lambda_k}{\sum_{k=1}^{s}\lambda_k} \ge \eta,
\qquad \eta \in (0,1).
\]

Define the patchwise POD reconstruction
\[
\hat{\mathbf{u}}_p
=
\bar{\mathbf{u}}
+
\Phi\Phi^\top
(\mathbf{u}_p - \bar{\mathbf{u}}),
\]
and let $\hat{\mathbf{u}}$ be the global reconstruction obtained by
reassembling the reconstructed patches. Then
\[
\mathbb{E}\!\left[\|\mathbf{u}-\hat{\mathbf{u}}\|^2\right]
\le
(1-\eta)\sum_{p=1}^{P}
\mathbb{E}\!\left[\|\mathbf{u}_p-\bar{\mathbf{u}}\|^2\right].
\]
\end{proposition}

\begin{proof}
Because the patches are non-overlapping,
\[
\|\mathbf{u}-\hat{\mathbf{u}}\|^2
=
\sum_{p=1}^{P}\|\mathbf{u}_p-\hat{\mathbf{u}}_p\|^2.
\]
Taking expectations gives
\[
\mathbb{E}\!\left[\|\mathbf{u}-\hat{\mathbf{u}}\|^2\right]
=
\sum_{p=1}^{P}
\mathbb{E}\!\left[\|\mathbf{u}_p-\hat{\mathbf{u}}_p\|^2\right].
\]

For each patch $p$, define the centered patch
$\tilde{\mathbf{u}}_p := \mathbf{u}_p - \bar{\mathbf{u}}$. Then
\[
\mathbf{u}_p-\hat{\mathbf{u}}_p
=
(\mathbf{I}-\Phi\Phi^\top)\tilde{\mathbf{u}}_p,
\]
and therefore
\[
\|\mathbf{u}_p-\hat{\mathbf{u}}_p\|^2
=
\|(\mathbf{I}-\Phi\Phi^\top)\tilde{\mathbf{u}}_p\|^2.
\]

Since $\Phi$ is constructed from the pooled patch covariance, the POD
projection identity applies to the pooled patch distribution:
\[
\frac{1}{P}\sum_{p=1}^{P}
\mathbb{E}\!\left[
\|(\mathbf{I}-\Phi\Phi^\top)\tilde{\mathbf{u}}_p\|^2
\right]
=
\sum_{k=K+1}^{s}\lambda_k.
\]
Using the retained-energy assumption,
\[
\sum_{k=K+1}^{s}\lambda_k
=
\sum_{k=1}^{s}\lambda_k
-
\sum_{k=1}^{K}\lambda_k
\le
(1-\eta)\sum_{k=1}^{s}\lambda_k.
\]
The total variance of the pooled centered patches is
\[
\sum_{k=1}^{s}\lambda_k
=
\frac{1}{P}\sum_{p=1}^{P}
\mathbb{E}\!\left[
\|\mathbf{u}_p-\bar{\mathbf{u}}\|^2
\right].
\]
Combining the previous two equations gives
\[
\frac{1}{P}\sum_{p=1}^{P}
\mathbb{E}\!\left[
\|\mathbf{u}_p-\hat{\mathbf{u}}_p\|^2
\right]
\le
(1-\eta)
\frac{1}{P}\sum_{p=1}^{P}
\mathbb{E}\!\left[
\|\mathbf{u}_p-\bar{\mathbf{u}}\|^2
\right].
\]
Multiplying both sides by $P$ yields
\[
\sum_{p=1}^{P}
\mathbb{E}\!\left[
\|\mathbf{u}_p-\hat{\mathbf{u}}_p\|^2
\right]
\le
(1-\eta)
\sum_{p=1}^{P}
\mathbb{E}\!\left[
\|\mathbf{u}_p-\bar{\mathbf{u}}\|^2
\right].
\]
Finally, using the non-overlapping patch decomposition,
\[
\mathbb{E}\!\left[\|\mathbf{u}-\hat{\mathbf{u}}\|^2\right]
\le
(1-\eta)\sum_{p=1}^{P}
\mathbb{E}\!\left[\|\mathbf{u}_p-\bar{\mathbf{u}}\|^2\right].
\]
This completes the proof.
\end{proof}

\begin{remark}
If $\eta = 0.99$, then the expected reconstruction error under the pooled
patch distribution is at most $1\%$ of the total pooled patch variance.
Equivalently, for a non-overlapping reconstruction, the expected global
reconstruction error is bounded by $1\%$ of the sum of patchwise variances.
\end{remark}
\section{Proposition 2 Proof}
\label{app:proof2}
\begin{proposition}[Analytic Uncertainty Propagation under Patchwise POD]
Let an input field be represented by $P$ patches. For each patch $p$,
let the reconstruction be
\[
\hat{\mathbf{u}}_p = \bar{\mathbf{u}} + \Phi \mathbf{a}_p,
\]
where $\bar{\mathbf{u}} \in \mathbb{R}^{s}$ is the global patch mean,
$\Phi \in \mathbb{R}^{s \times K}$ is the shared retained POD basis,
and $\mathbf{a}_p \in \mathbb{R}^{K}$ is a random latent coefficient vector
with covariance $\Sigma_{a_p} := \mathrm{Cov}(\mathbf{a}_p)$.

Then the following hold:

\begin{enumerate}
    \item[\textbf{(i)}] \textbf{Patch-level covariance.}
    The covariance of the reconstructed patch is
    \[
    \Sigma_{u_p}
    :=
    \mathrm{Cov}(\hat{\mathbf{u}}_p)
    =
    \Phi \Sigma_{a_p} \Phi^\top.
    \]

    \item[\textbf{(ii)}] \textbf{Non-overlapping patch assembly.}
    Suppose patches are reassembled without overlap into
    $\hat{\mathbf{u}} = [\hat{\mathbf{u}}_1^\top \cdots \hat{\mathbf{u}}_P^\top]^\top$.
    If the latent variables are mutually uncorrelated across distinct patches,
    i.e.\ $\mathrm{Cov}(\mathbf{a}_p, \mathbf{a}_q) = \mathbf{0}$ for $p \neq q$,
    then the global covariance is block-diagonal:
    \[
    \Sigma_u
    :=
    \mathrm{Cov}(\hat{\mathbf{u}})
    =
    \mathrm{blkdiag}
    \left(
    \Phi \Sigma_{a_1} \Phi^\top,
    \dots,
    \Phi \Sigma_{a_P} \Phi^\top
    \right).
    \]

    \item[\textbf{(iii)}] \textbf{Overlapping patches with linear aggregation.}
    Suppose overlapping reconstructed patches are concatenated into
    $\tilde{\mathbf{u}} = [\hat{\mathbf{u}}_1^\top \cdots \hat{\mathbf{u}}_P^\top]^\top$,
    and the final field is obtained by a fixed linear operator $\mathcal{S}$:
    $\hat{\mathbf{u}} = \mathcal{S}\tilde{\mathbf{u}}$.
    If patch latents are mutually uncorrelated across patches, then
    \[
    \mathrm{Cov}(\tilde{\mathbf{u}})
    =
    \tilde{\Sigma}
    :=
    \mathrm{blkdiag}
    \left(
    \Phi \Sigma_{a_1} \Phi^\top,
    \dots,
    \Phi \Sigma_{a_P} \Phi^\top
    \right),
    \]
    and the global covariance is
    \[
    \Sigma_u = \mathcal{S}\tilde{\Sigma}\mathcal{S}^\top.
    \]
\end{enumerate}
\end{proposition}

\begin{proof}
The result follows from linear covariance propagation.
For each patch, $\hat{\mathbf{u}}_p = \bar{\mathbf{u}} + \Phi \mathbf{a}_p$.
Since $\bar{\mathbf{u}}$ is deterministic,
\[
\mathrm{Cov}(\hat{\mathbf{u}}_p)
=
\mathrm{Cov}(\Phi \mathbf{a}_p)
=
\Phi\, \mathrm{Cov}(\mathbf{a}_p)\,\Phi^\top,
\]
proving part (i).

Parts (ii) and (iii) follow by writing the global reconstruction as either
a concatenation (non-overlapping case) or a linear transformation via
$\mathcal{S}$ (overlapping case), and applying the identity
$\mathrm{Cov}(\mathbf{B}\mathbf{x}) = \mathbf{B}\,\mathrm{Cov}(\mathbf{x})\,\mathbf{B}^\top$,
together with the assumption that cross-patch covariances vanish.
\end{proof}

\begin{remark}[Practical significance]
Proposition~\ref{prop:patch_uq} shows that uncertainty propagation through the
patchwise POD decoder is analytically tractable due to its linear structure.
Unlike nonlinear decoders (e.g., VAEs or diffusion models in pixel space),
no Jacobian linearisation or Monte Carlo backpropagation is required.
\end{remark}

\begin{remark}[Empirical covariance estimation]
\label{rem:empirical_cov}
In practice, $\Sigma_{a_p}$ is estimated from $M$ sampled latent
realisations $\{\mathbf{a}_p^{(m)}\}_{m=1}^M$ using the empirical covariance
\[
\hat{\Sigma}_{a_p}
=
\frac{1}{M-1}
\sum_{m=1}^M
(\mathbf{a}_p^{(m)}-\bar{\mathbf{a}}_p)
(\mathbf{a}_p^{(m)}-\bar{\mathbf{a}}_p)^\top.
\]
The plug-in estimator
$\hat{\Sigma}_{u_p} = \Phi \hat{\Sigma}_{a_p}\Phi^\top$
is then used to construct analytic uncertainty maps.
\end{remark}

\begin{corollary}[Pointwise predictive variance]
\label{cor:pointwise_variance}
Under Proposition~\ref{prop:patch_uq}, the predictive variance at within-patch pixel index $\ell$ in patch $p$ is 
$[\Sigma_{u_p}]_{\ell\ell} = [\Phi\Sigma_{a_p}\Phi^\top]_{\ell\ell}$.
In the overlapping case, for global pixel index $\ell$,
$\mathrm{Var}[\hat{u}_\ell] = [\mathcal{S}\tilde{\Sigma}\mathcal{S}^\top]_{\ell\ell}$.
\end{corollary}




\section{POD Basis Analysis}
\label{app:pod_analysis}
\subsection{Singular Value Decay and Compression}
\label{app:svd_decay}

Figure~\ref{fig:cumulative_energy} shows the cumulative retained energy of the patchwise POD basis across datasets. The shaded bands indicate variation across patches, and the dashed line marks the $99\%$ energy threshold used in the main experiments.

The required number of modes varies noticeably: $K_{99}=2$ for SST, $K_{99}=11$ for Chest X-ray, and $K_{99}=26$ for FFHQ. This reflects increasing local complexity from smooth geophysical fields to anatomical structure and natural image texture. The ordering also explains the observed compression trends: datasets with faster spectral decay allow stronger dimensionality reduction, while more complex datasets require larger latent dimensions to preserve fine-scale structure.





\subsection{POD Mode Statistics}
\label{app:svd_stats}

Table~\ref{tab:svd_stats} reports the mean, median, and maximum number of POD modes needed to retain $95\%$ and $99\%$ variance across patches. The median $K_{99}$ values match those used in the main experiments, showing that the selected latent dimensions are representative. SST and Chest X-ray show small mean--median gaps, indicating uniform spectral decay, while FFHQ shows a mild tail of slower-decaying patches due to higher texture variability. Overall, the statistics support a shared patchwise POD basis whose dimensionality remains stable within each dataset while adapting to domain complexity.
\begin{table}[htbp]
\centering
\caption{POD mode statistics per dataset. Mean, median, and maximum number of modes required to retain $95\%$ and $99\%$ of patchwise variance, computed across all patches within each dataset.}
\label{tab:svd_stats}
\footnotesize
\setlength{\tabcolsep}{5pt}
\begin{tabular}{l ccc ccc}
\toprule
& \multicolumn{3}{c}{$\eta = 0.95$} & \multicolumn{3}{c}{$\eta = 0.99$} \\
\cmidrule(lr){2-4} \cmidrule(lr){5-7}
Dataset & Mean & Median & Max & Mean & Median & Max \\
\midrule
SST        & 1.005 & 1 & 2  & 1.594  & 2  & 3  \\
ChestX-ray & 3.078 & 3 & 4  & 11.551 & 11 & 13 \\
FFHQ       & 5.195 & 5 & 8  & 26.875 & 26 & 32 \\
\bottomrule
\end{tabular}
\end{table}
\begin{figure}[htbp]
  \centering
  \includegraphics[width=0.85\columnwidth]{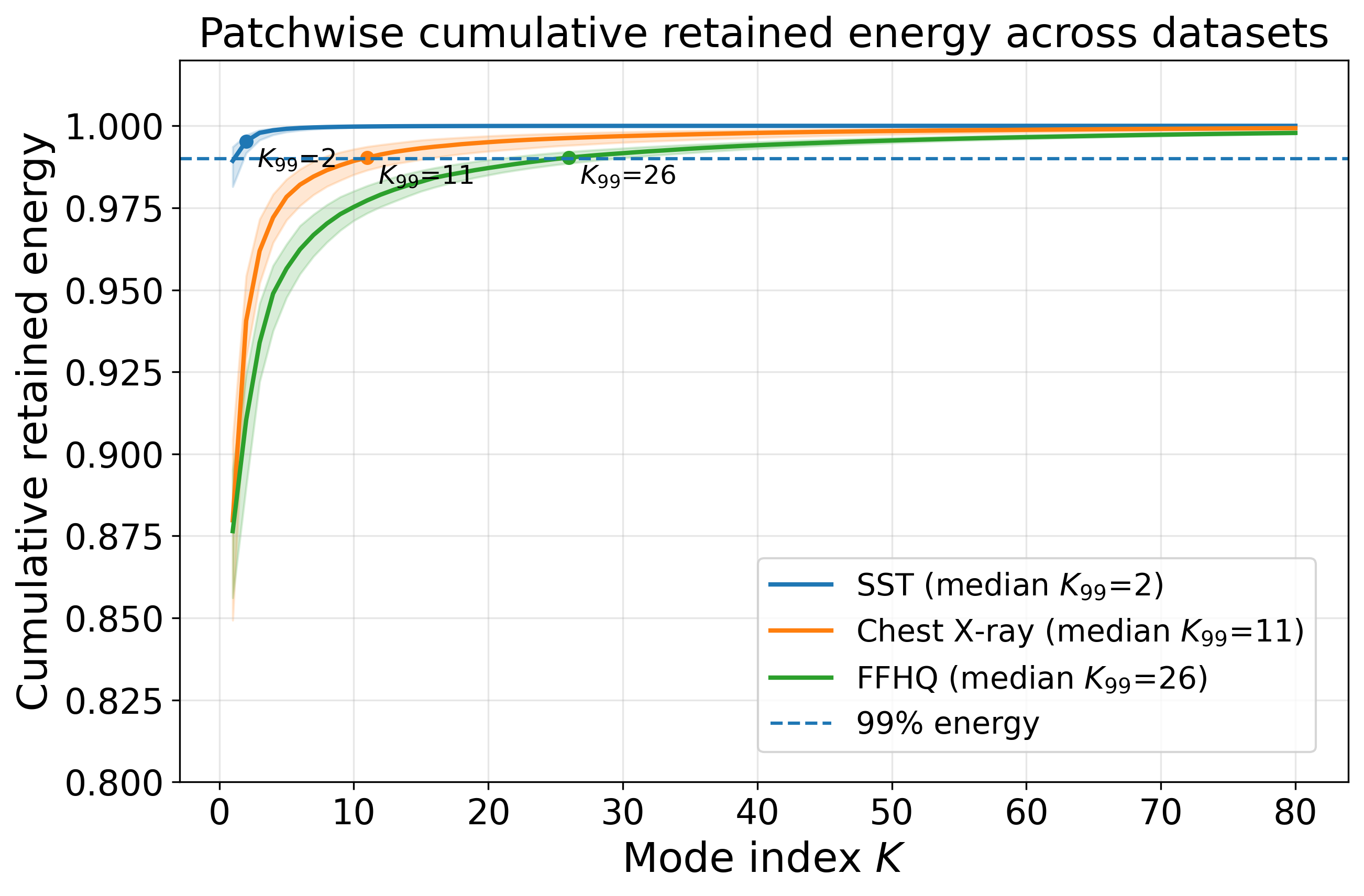}
  \caption{Patchwise cumulative retained energy as a function of mode index $K$ for each dataset. Shaded bands indicate variability across patches. The dashed line marks the $\eta = 0.99$ energy threshold. SST reaches $K_{99}=2$, ChestX-ray $K_{99}=11$, and FFHQ $K_{99}=26$, reflecting differing degrees of local low-rank structure across domains.}
  \label{fig:cumulative_energy}
\end{figure}
%
\subsection{Validation of the POD Reconstruction Bound}
We validate the reconstruction error bound in Proposition~\ref{prop:patch_error} by comparing the empirical POD reconstruction error with the theoretical bound for each dataset. Table~\ref{tab:prop1_validation} shows that the mean empirical POD reconstruction error stays below the expected bound for all datasets, with ratios of 0.94--0.98; sample-wise satisfaction is reported as an empirical diagnostic, not as a deterministic guarantee. This supports the retained-energy criterion in Eq.~\ref{eq:energy} as a reliable control on patchwise reconstruction error.
\begin{table}[b]
\centering
\caption{Empirical validation of Proposition~\ref{prop:patch_error}. The mean reconstruction error remains below the expected theoretical bound across all datasets.}
\label{tab:prop1_validation}
\scriptsize
\setlength{\tabcolsep}{6pt}
\begin{tabular}{lcccc}
\toprule
Dataset & Bound & Actual Error & Ratio & \% Satisfying \\
\midrule
SST        & 2.6075  & 2.4800  & 0.9810 & 99.75 \\
X-ray      & 2.6521 & 2.5980 & 0.9511 & 98.95 \\
FFHQ       & 23.6347 & 23.1856 & 0.9796 & 97.62 \\
\bottomrule
\end{tabular}
\end{table}
%
\section{Extended Reconstruction Results}
\label{app:extended_results}
Table~\ref{tab:model_comparison_full} extends Table~\ref{tab:model_comparison} by reporting mean $\pm$ standard deviation across three seed runs for RMSE, PSNR, SSIM, and LPIPS across all models and datasets. FID is reported once per dataset, since it is a distribution-level metric rather than a per-sample statistic.
\begin{table*}[t]
\centering
\caption{Reconstruction quality across three datasets. Lower values indicate better performance for
RMSE, LPIPS, and FID, while higher values are better for PSNR and SSIM. Results are reported as
mean $\pm$ standard deviation across three independent training runs, where each run first averages
the metric over the test set. FID is computed at the dataset level and reported once.}
\label{tab:model_comparison_full}
\scriptsize
\setlength{\tabcolsep}{3pt}
\begin{tabular}{l c c c c c c}
\toprule
Metric & \textbf{Patch-PODiff-ViT} & VAE-LDM & DiT & PixelDiff & U-Net & Fullfield-PODiff \\
\midrule
\multicolumn{7}{c}{\textit{SST (Sea Surface Temperature)}} \\
\midrule
RMSE $\downarrow$  & $\mathbf{0.0030 \pm 0.0010}$ & $0.0041 \pm 0.0016$ & $0.0040 \pm 0.0018$ & $0.0049 \pm 0.0017$ & $0.0093 \pm 0.0051$ & $0.0049 \pm 0.0018$ \\
PSNR $\uparrow$    & $\mathbf{50.43 \pm 1.09}$    & $47.73 \pm 1.97$    & $47.95 \pm 1.09$    & $46.23 \pm 2.57$    & $40.63 \pm 2.11$    & $46.22 \pm 1.94$    \\
SSIM $\uparrow$    & $\mathbf{0.9888 \pm 0.0043}$ & $0.9799 \pm 0.0100$ & $0.9707 \pm 0.0042$ & $0.9713 \pm 0.0110$ & $0.9589 \pm 0.0255$ & $0.9690 \pm 0.0093$ \\
LPIPS $\downarrow$ & $\mathbf{0.0131 \pm 0.0064}$ & $0.0201 \pm 0.0101$ & $0.0169 \pm 0.0058$ & $0.0234 \pm 0.0098$ & $0.0295 \pm 0.0025$ & $0.0267 \pm 0.0109$ \\
FID $\downarrow$   & $\mathbf{3.986}$             & $5.90$              & $5.01$              & $8.46$              & $11.96$             & $8.49$              \\
\midrule
\multicolumn{7}{c}{\textit{Chest X-ray}} \\
\midrule
RMSE $\downarrow$  & $\mathbf{0.0065 \pm 0.0011}$ & $0.0087 \pm 0.0046$ & $0.0081 \pm 0.0019$ & $0.0098 \pm 0.0099$ & $0.0133 \pm 0.0059$ & $0.0092 \pm 0.0057$ \\
PSNR $\uparrow$    & $\mathbf{42.98 \pm 1.31}$    & $39.11 \pm 1.98$    & $41.21 \pm 1.31$    & $37.64 \pm 2.85$    & $34.52 \pm 2.11$    & $35.75 \pm 2.11$    \\
SSIM $\uparrow$    & $\mathbf{0.9885 \pm 0.0030}$ & $0.9790 \pm 0.0100$ & $0.9632 \pm 0.0042$ & $0.9743 \pm 0.0104$ & $0.9423 \pm 0.0156$ & $0.9699 \pm 0.0084$ \\
LPIPS $\downarrow$ & $\mathbf{0.0201 \pm 0.0079}$ & $0.0298 \pm 0.0086$ & $0.0281 \pm 0.0094$ & $0.0310 \pm 0.0101$ & $0.0422 \pm 0.0106$ & $0.0307 \pm 0.0100$ \\
FID $\downarrow$   & $\mathbf{6.0152}$             & $8.17$              & $7.01$              & $11.44$             & $13.65$             & $9.26$              \\
\midrule
\multicolumn{7}{c}{\textit{FFHQ (Face Images)}} \\
\midrule
RMSE $\downarrow$  & $\mathbf{0.0109 \pm 0.0081}$ & $0.0178 \pm 0.0080$ & $0.0121 \pm 0.0122$ & $0.0147 \pm 0.0210$ & $0.0361 \pm 0.0097$ & $0.0154 \pm 0.0143$ \\
PSNR $\uparrow$    & $\mathbf{39.15 \pm 2.33}$    & $36.44 \pm 2.09$    & $38.04 \pm 2.33$    & $37.87 \pm 1.99$    & $32.54 \pm 2.56$    & $37.15 \pm 2.21$    \\
SSIM $\uparrow$    & $\mathbf{0.9522 \pm 0.0096}$ & $0.9321 \pm 0.0079$ & $0.9432 \pm 0.0042$ & $0.9411 \pm 0.0099$ & $0.9255 \pm 0.0236$ & $0.9358 \pm 0.0500$ \\
LPIPS $\downarrow$ & $\mathbf{0.0300 \pm 0.0056}$ & $0.0561 \pm 0.0056$ & $0.0377 \pm 0.0067$ & $0.0427 \pm 0.0068$ & $0.0584 \pm 0.0552$ & $0.0471 \pm 0.0366$ \\
FID $\downarrow$   & $\mathbf{9.168}$             & $11.98$             & $10.04$             & $10.11$             & $16.18$             & $12.08$             \\
\bottomrule
\end{tabular}
\end{table*}
\section{Qualitative Comparisons}
\label{app:qualitative}
\begin{figure}[htbp]
  \centering
  \includegraphics[width=\columnwidth]{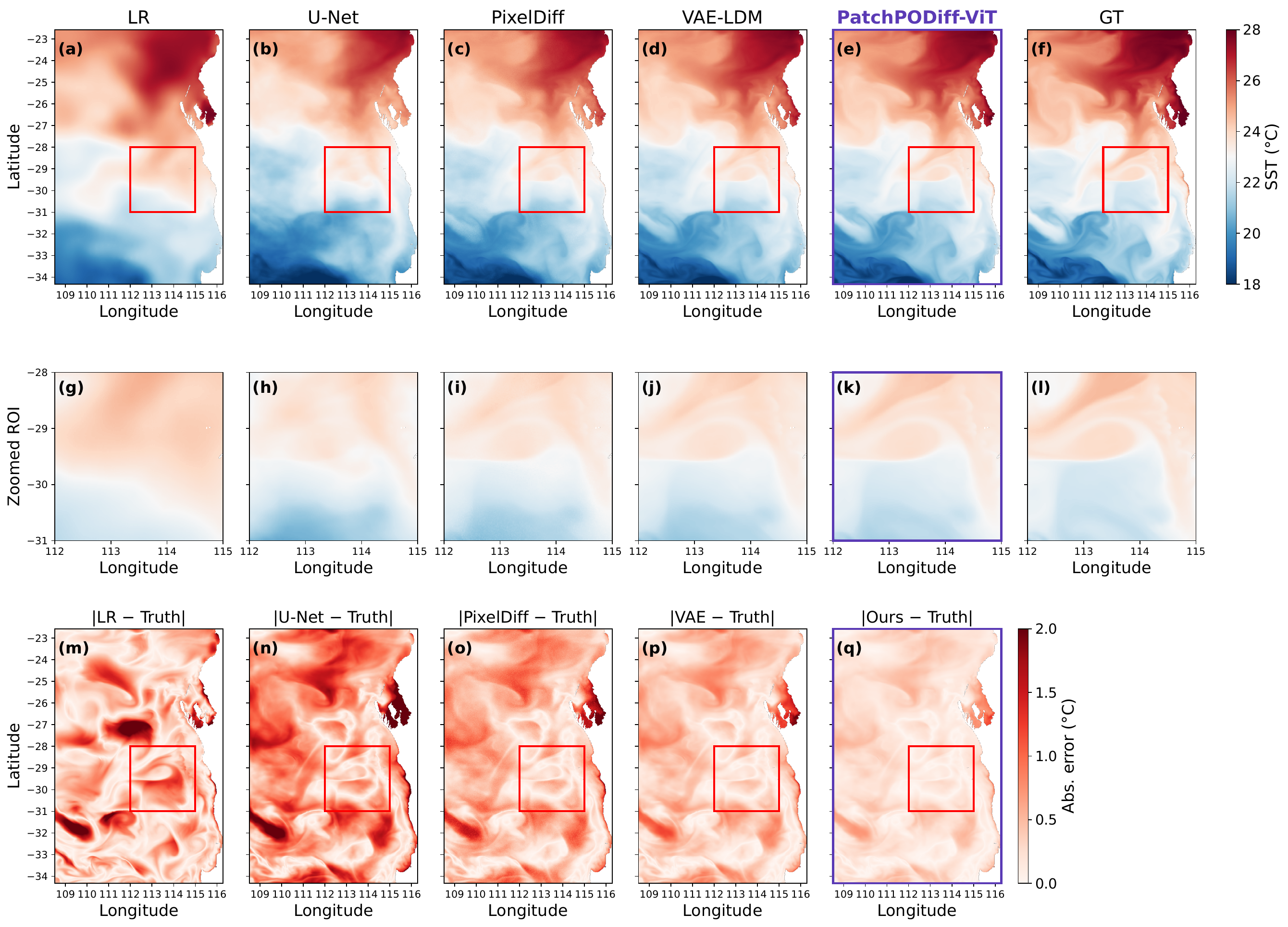}
  \caption{Qualitative comparison on SST. Columns show upscaled LR input, U-Net, PixelDiff, VAE-LDM, Patch-PODiff-ViT, and ground truth (GT). Top row shows full-field reconstructions, middle row shows a zoomed region of interest, and bottom row shows absolute reconstruction error. Patch-PODiff-ViT consistently preserves both large-scale structure and fine-scale variability, particularly in regions with strong spatial gradients.}
  \label{fig:sst_appendix_comparison}
\end{figure}
Figures~\ref{fig:sst_appendix_comparison}, \ref{fig:xray_comparison}, and~\ref{fig:ffhq_comparison} present extended qualitative super-resolution results for SST, Chest X-ray, and FFHQ respectively, complementing the  three-domain overview in Figure~\ref{fig:qualitative_comparison}. The SST figure shows one representative scene; the X-ray and FFHQ figures each show two scenes. All include full reconstructions, zoomed regions of interest, and corresponding error maps, consistent with the quantitative trends reported in Section~\ref{sec:reconstruction_accuracy}.

On SST (Figure~\ref{fig:sst_appendix_comparison}), the low-resolution input and U-Net baseline fail to recover fine-scale thermal fronts, particularly in the zoomed coastal region. PixelDiff and VAE-LDM partially recover large-scale structure but exhibit smoothing near sharp meridional gradients. Patch-PODiff-ViT produces the lowest error, preserving both the large-scale warm water mass and fine-scale frontal structure in the zoomed region.

On Chest X-ray (Figure~\ref{fig:xray_comparison}), the low-resolution input and U-Net baseline exhibit higher reconstruction errors than diffusion-based methods, particularly along rib edges and lung boundaries. PixelDiff and VAE-LDM reduce these errors but introduce noticeable smoothing of anatomical structures. In contrast, Patch-PODiff-ViT produces sharper boundaries and more accurate structural detail, with consistently lower error in both global views and zoomed regions.

On FFHQ (Figure~\ref{fig:ffhq_comparison}), baseline methods tend to smooth or distort high-frequency textures, especially in regions such as hair and eyes. Patch-PODiff-ViT better preserves fine-scale detail, resulting in more faithful reconstructions and reduced error in high-detail regions, as seen in the zoomed panels.
\begin{figure}[htbp]
  \centering
  \includegraphics[width=\columnwidth]{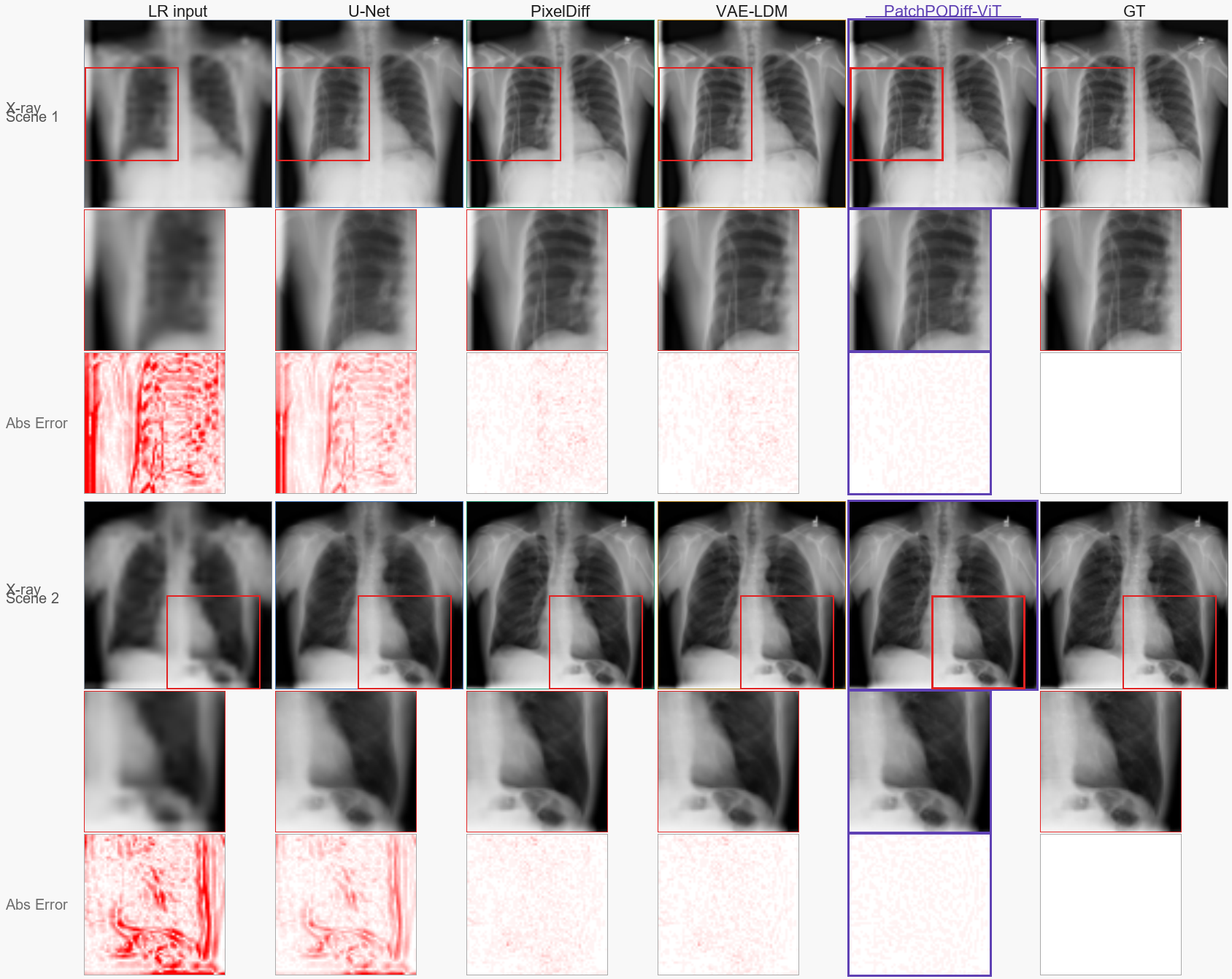}
  \caption{Qualitative comparison on Chest X-ray. Each scene shows (top) full reconstruction, (middle) zoomed region, and (bottom) absolute error. Columns correspond to LR input, U-Net, PixelDiff, VAE-LDM, Patch-PODiff-ViT, and ground truth. Patch-PODiff-ViT recovers sharper anatomical structures with lower reconstruction error, particularly along rib edges and lung boundaries.}
  \label{fig:xray_comparison}
\end{figure}

\begin{figure}[htbp]
  \centering
  \includegraphics[width=\columnwidth]{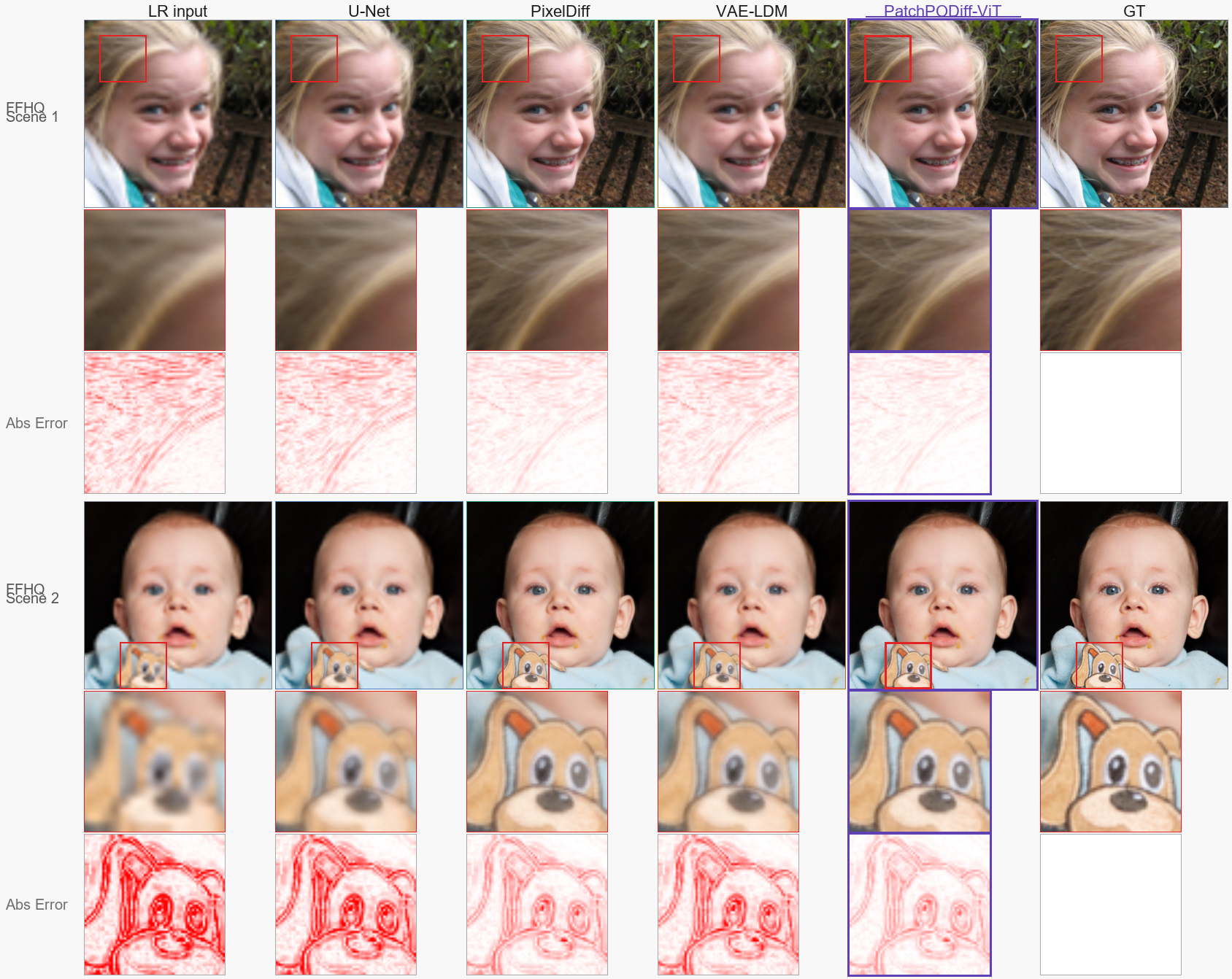}
  \caption{Qualitative comparison on FFHQ. Each scene shows (top) full reconstruction, (middle) zoomed region, and (bottom) absolute error. Columns correspond to LR input, U-Net, PixelDiff, VAE-LDM, Patch-PODiff-ViT, and ground truth. Patch-PODiff-ViT preserves fine-scale texture more effectively, with lower error in high-detail regions such as hair and facial features (Scene 1). Gains are more modest in regions with sharp synthetic boundaries such as the cartoon detail in Scene 2.}
  \label{fig:ffhq_comparison}
\end{figure}
\section{Coverage and Calibration Details}
\label{app:coverage}
\subsection{CRPS and MACE Summary}
\label{app:crps_mace}
Table~\ref{tab:crps_mace} reports CRPS and MACE for all models using $M=100$ ensemble samples, with SST averaged over all 2011 test days. Patch-PODiff-ViT achieves the lowest CRPS on SST and Chest X-ray, indicating sharper predictive distributions in these domains. On FFHQ, VAE-LDM attains a lower CRPS ($0.00892$) than Patch-PODiff-ViT ($0.0116$), but this comes with a higher MACE ($0.0121$ vs.\ $0.0084$), indicating improved sharpness at the cost of reduced calibration. Across all datasets, Patch-PODiff-ViT consistently achieves the lowest MACE, demonstrating better calibration of predictive uncertainty.
\begin{table}[t]
\centering
\caption{CRPS (mean $\pm$ standard deviation) and MACE across datasets. Lower values indicate better performance. CRPS is computed over $M=100$ ensemble samples, while MACE is evaluated against nominal coverage levels.}
\label{tab:crps_mace}
\footnotesize
\setlength{\tabcolsep}{4pt}
\begin{tabular}{l cc cc cc}
\toprule
& \multicolumn{2}{c}{SST} & \multicolumn{2}{c}{Chest X-ray} & \multicolumn{2}{c}{FFHQ} \\
\cmidrule(lr){2-3}\cmidrule(lr){4-5}\cmidrule(lr){6-7}
Model & CRPS & MACE & CRPS & MACE & CRPS & MACE \\
\midrule
Patch-PODiff-ViT & $0.00186\,(0.00258)$ & \textbf{0.0080} & $0.00360\,(0.00060)$ & \textbf{0.0046} & $0.01160\,(0.00160)$ & \textbf{0.0084} \\
DiT              & $0.00270\,(0.00325)$ & 0.0216 & $0.00498\,(0.00115)$ & 0.0099 & $0.02680\,(0.00548)$ & 0.0143 \\
VAE-LDM          & $0.00206\,(0.00399)$ & 0.0159 & $0.01520\,(0.00235)$ & 0.0114 & $0.00892\,(0.00109)$ & 0.0121 \\
PixelDiff        & $0.01900\,(0.00650)$ & 0.0193 & $0.04890\,(0.00658)$ & 0.0206 & $0.01009\,(0.00210)$ & 0.0281 \\
\bottomrule
\end{tabular}
\end{table}

\subsection{Full Coverage Tables}
\label{app:coverage_full}
Tables~\ref{tab:coverage_sst}--\ref{tab:coverage_ffhq} report empirical coverage at nominal levels $\{50, 60, 70, 80, 90, 95, 99\}\%$ for all models across SST, Chest X-ray, and FFHQ. For each level, we report the empirical coverage along with the deviation $\Delta = \mathrm{empirical} - \mathrm{nominal}$, where positive values indicate over-coverage and negative values indicate under-coverage.
Across all datasets, Patch-PODiff-ViT exhibits the smallest deviations from nominal coverage, indicating consistently well-calibrated uncertainty estimates. In contrast, DiT and VAE-LDM tend to show moderate over- or under-coverage depending on the dataset, while PixelDiff exhibits larger deviations, particularly at higher confidence levels.
\begin{table}[t]
\centering
\caption{Empirical coverage at each nominal level for SST. $\Delta = \mathrm{empirical} - \mathrm{nominal}$; positive values indicate over-coverage.}
\label{tab:coverage_sst}
\footnotesize
\setlength{\tabcolsep}{4pt}
\begin{tabular}{l rr rr rr rr}
\toprule
& \multicolumn{2}{c}{Patch-PODiff-ViT} & \multicolumn{2}{c}{DiT}
& \multicolumn{2}{c}{VAE-LDM} & \multicolumn{2}{c}{PixelDiff} \\
\cmidrule(lr){2-3}\cmidrule(lr){4-5}\cmidrule(lr){6-7}\cmidrule(lr){8-9}
Nominal & Empirical & $\Delta$ & Empirical & $\Delta$ & Empirical & $\Delta$ & Empirical & $\Delta$ \\
\midrule
50\% & 0.4882 & $-0.012$ & 0.4662 & $-0.034$ & 0.4645 & $-0.036$ & 0.4918 & $-0.008$ \\
60\% & 0.5804 & $-0.020$ & 0.5700 & $-0.030$ & 0.5690 & $-0.031$ & 0.5952 & $-0.005$ \\
70\% & 0.7058 & $+0.006$ & 0.7176 & $+0.018$ & 0.6928 & $-0.007$ & 0.6808 & $-0.019$ \\
80\% & 0.8058 & $+0.006$ & 0.8226 & $+0.023$ & 0.8139 & $+0.014$ & 0.8022 & $+0.002$ \\
90\% & 0.8999 & $-0.000$ & 0.8883 & $-0.012$ & 0.8905 & $-0.010$ & 0.9085 & $+0.009$ \\
95\% & 0.9510 & $+0.001$ & 0.9370 & $-0.013$ & 0.9617 & $+0.012$ & 0.9150 & $-0.035$ \\
99\% & 0.9918 & $+0.002$ & 0.9726 & $-0.017$ & 0.9878 & $-0.002$ & 0.9326 & $-0.057$ \\
\bottomrule
\end{tabular}
\end{table}

\begin{table}[t]
\centering
\caption{Empirical coverage at each nominal level for Chest X-ray.}
\label{tab:coverage_xray}
\footnotesize
\setlength{\tabcolsep}{4pt}
\begin{tabular}{l rr rr rr rr}
\toprule
& \multicolumn{2}{c}{Patch-PODiff-ViT} & \multicolumn{2}{c}{DiT}
& \multicolumn{2}{c}{VAE-LDM} & \multicolumn{2}{c}{PixelDiff} \\
\cmidrule(lr){2-3}\cmidrule(lr){4-5}\cmidrule(lr){6-7}\cmidrule(lr){8-9}
Nominal & Empirical & $\Delta$ & Empirical & $\Delta$ & Empirical & $\Delta$ & Empirical & $\Delta$ \\
\midrule
50\% & 0.5089 & $+0.009$ & 0.5171 & $+0.017$ & 0.5169 & $+0.017$ & 0.5001 & $+0.000$ \\
60\% & 0.6072 & $+0.007$ & 0.6002 & $+0.000$ & 0.6178 & $+0.018$ & 0.6302 & $+0.030$ \\
70\% & 0.6980 & $-0.002$ & 0.6914 & $-0.009$ & 0.7081 & $+0.008$ & 0.6687 & $-0.031$ \\
80\% & 0.8033 & $+0.003$ & 0.8106 & $+0.011$ & 0.8149 & $+0.015$ & 0.7919 & $-0.008$ \\
90\% & 0.9088 & $+0.009$ & 0.9170 & $+0.017$ & 0.9023 & $+0.002$ & 0.9350 & $+0.035$ \\
95\% & 0.9493 & $-0.001$ & 0.9404 & $-0.010$ & 0.9400 & $-0.010$ & 0.9654 & $+0.015$ \\
99\% & 0.9911 & $+0.001$ & 0.9962 & $+0.006$ & 0.9798 & $-0.010$ & 0.9659 & $-0.024$ \\
\bottomrule
\end{tabular}
\end{table}

\begin{table}[t]
\centering
\caption{Empirical coverage at each nominal level for FFHQ.}
\label{tab:coverage_ffhq}
\footnotesize
\setlength{\tabcolsep}{4pt}
\begin{tabular}{l rr rr rr rr}
\toprule
& \multicolumn{2}{c}{Patch-PODiff-ViT} & \multicolumn{2}{c}{DiT}
& \multicolumn{2}{c}{VAE-LDM} & \multicolumn{2}{c}{PixelDiff} \\
\cmidrule(lr){2-3}\cmidrule(lr){4-5}\cmidrule(lr){6-7}\cmidrule(lr){8-9}
Nominal & Empirical & $\Delta$ & Empirical & $\Delta$ & Empirical & $\Delta$ & Empirical & $\Delta$ \\
\midrule
50\% & 0.4889 & $-0.011$ & 0.4954 & $-0.005$ & 0.4827 & $-0.017$ & 0.5144 & $+0.014$ \\
60\% & 0.6172 & $+0.017$ & 0.6257 & $+0.026$ & 0.6264 & $+0.026$ & 0.6657 & $+0.066$ \\
70\% & 0.6990 & $-0.001$ & 0.7062 & $+0.006$ & 0.7080 & $+0.008$ & 0.7342 & $+0.034$ \\
80\% & 0.8103 & $+0.010$ & 0.8196 & $+0.020$ & 0.8039 & $+0.004$ & 0.7919 & $-0.008$ \\
90\% & 0.9091 & $+0.009$ & 0.9173 & $+0.017$ & 0.9204 & $+0.020$ & 0.9500 & $+0.050$ \\
95\% & 0.9411 & $-0.009$ & 0.9319 & $-0.018$ & 0.9500 & $+0.000$ & 0.9649 & $+0.015$ \\
99\% & 0.9891 & $-0.001$ & 0.9817 & $-0.008$ & 0.9816 & $-0.008$ & 0.9997 & $+0.010$ \\
\bottomrule
\end{tabular}
\end{table}

\section{Uncertainty Map Comparisons}
\label{app:uq_maps}
\begin{figure}[t]
  \centering
  \includegraphics[width=\columnwidth]{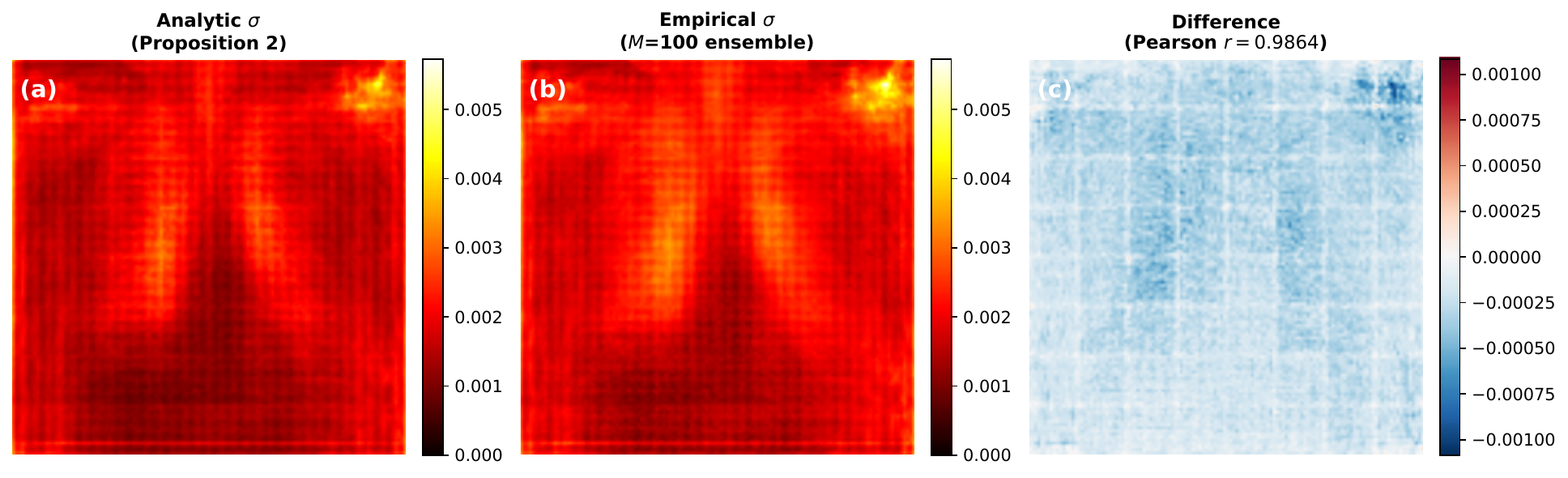}
  \caption{Analytic versus empirical uncertainty on Chest X-ray. Left: analytic standard deviation from Proposition~\ref{prop:patch_uq}. Middle: empirical standard deviation computed from $M=100$ ensemble samples. Right: difference map with Pearson correlation $r = 0.986$. High uncertainty is concentrated along lung boundaries and rib edges.}
  \label{fig:uq_xray}
\end{figure}

\begin{figure}[t]
  \centering
  \includegraphics[width=\columnwidth]{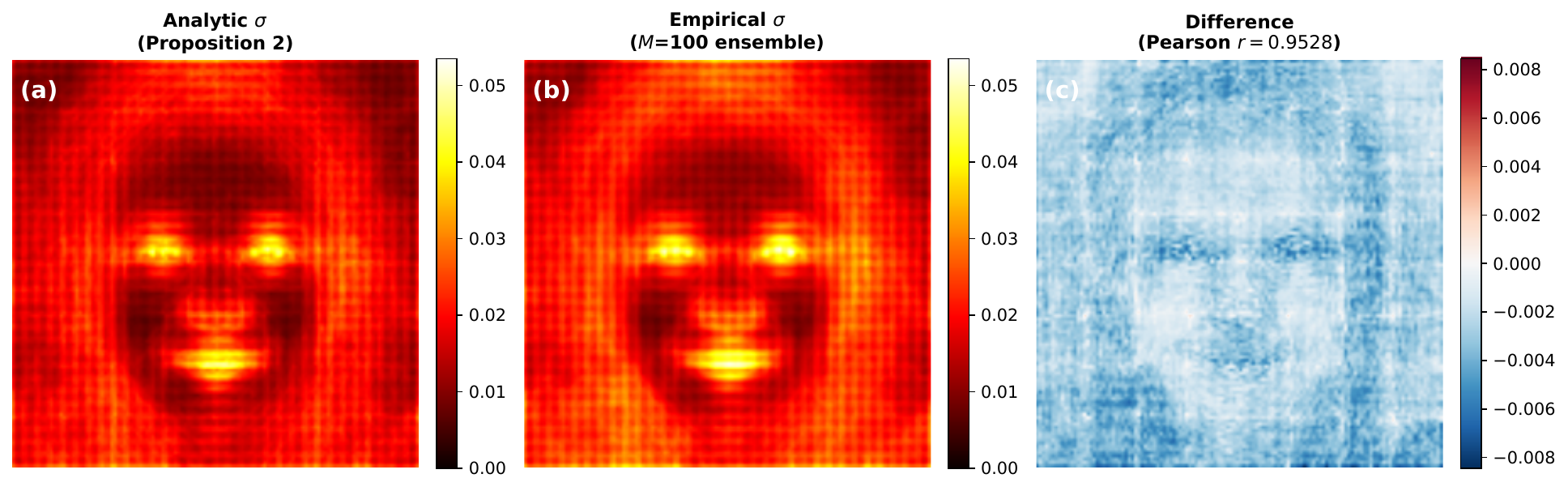}
  \caption{Analytic versus empirical uncertainty on FFHQ. Left: analytic standard deviation from Proposition~\ref{prop:patch_uq}. Middle: empirical standard deviation computed from $M=100$ ensemble samples. Right: difference map with Pearson correlation $r = 0.953$. The faint grid pattern reflects patch-boundary effects from independent patchwise covariance propagation and deterministic stitching.}
  \label{fig:uq_ffhq}
\end{figure}
Figures~\ref{fig:uq_xray} and~\ref{fig:uq_ffhq} extend the analytic versus empirical uncertainty comparison from SST (Figure~\ref{fig:uq_maps_sst}) to Chest X-ray and FFHQ, further supporting the calibration results in Section~\ref{sec:results_uq}.

For Chest X-ray, analytic and empirical maps agree strongly ($r=0.986$), with high uncertainty concentrated along lung boundaries and rib edges, where structural variability and reconstruction difficulty are greatest. For FFHQ, agreement remains high ($r=0.9528$), with uncertainty concentrated around fine-scale features such as eyes, mouth, and hair.
The faint grid pattern in the FFHQ difference map reflects patch-boundary effects from independent patchwise covariance propagation and deterministic stitching. While cross-patch covariance affects off-diagonal entries of the global covariance, pointwise variances are determined by the propagated within-patch covariance and any aggregation or post-processing applied after decoding. These effects are most visible for FFHQ, where texture varies strongly across neighboring patches.

\begin{figure}[t]
  \centering
  \includegraphics[width=\columnwidth]{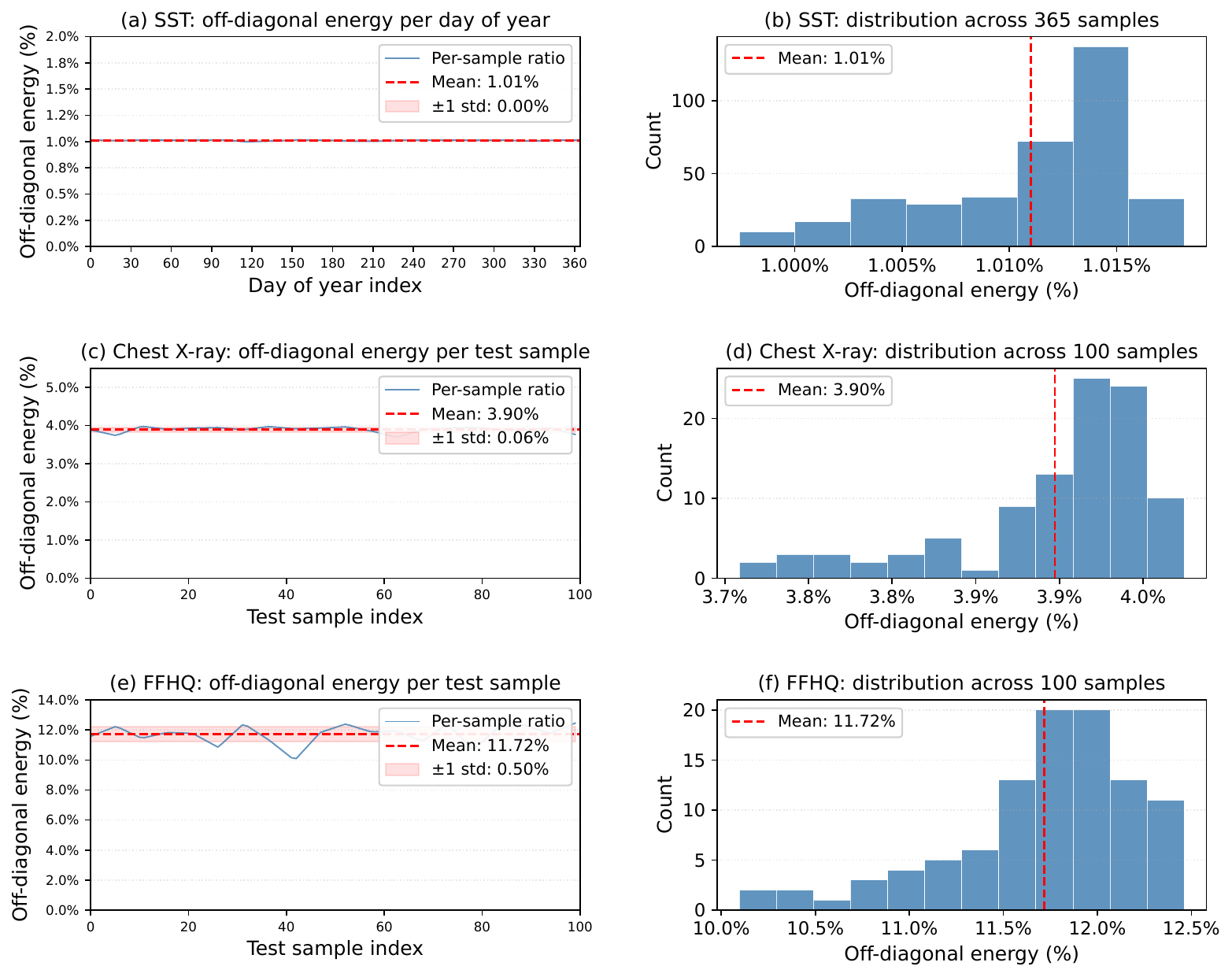}
  \caption{Cross-patch covariance energy. SST and Chest X-ray show very low off-diagonal energy, while FFHQ shows a moderate but still limited ratio, consistent with stronger texture-driven cross-patch dependence.}
  \label{fig:cross_patch_covariance}
\end{figure}
%
%

\section{Cross-Patch Covariance Analysis}
\label{app:cross_patch_covariance}
Proposition~\ref{prop:patch_uq} uses a block-diagonal covariance approximation that neglects cross-patch covariance. To quantify this, we measure the fraction of latent covariance energy in off-diagonal cross-patch blocks.
Figure~\ref{fig:cross_patch_covariance} shows low off-diagonal energy for SST and Chest X-ray, with mean ratios of $1.01\%$ and $3.90\%$, respectively, and a moderate but still limited ratio for FFHQ ($11.72\%$). This supports the block-diagonal approximation for structured physical and medical fields, while explaining the mild FFHQ patch-boundary discrepancies observed in uncertainty maps.

\section{Ablation Studies}
\label{app:Other_ablations}



\subsection{Energy Threshold}
\label{app:eta_ablation}
Table~\ref{tab:eta_ablation} evaluates $\eta \in \{0.99,0.95,0.90\}$. Lower $\eta$ reduces retained modes and latent dimensionality, but degrades reconstruction. Pixel metrics change moderately from $\eta=0.99$ to $0.95$, whereas FID worsens sharply at $\eta=0.90$ (e.g., $3.986 \rightarrow 12.14$ on SST and $9.17 \rightarrow 27.92$ on FFHQ), supporting $\eta=0.99$ as the default.
\begin{table}[htbp]
\centering
\caption{Effect of energy threshold $\eta$ on reconstruction quality. Lower is better for RMSE, LPIPS, FID; higher for PSNR, SSIM.}
\label{tab:eta_ablation}
\footnotesize
\setlength{\tabcolsep}{4pt}
\begin{tabular}{l l ccccc}
\toprule
Dataset & $\eta$ & RMSE & PSNR & SSIM & LPIPS & FID \\
\midrule
SST & 0.99 (default) & \textbf{0.0030} & \textbf{50.43} & \textbf{0.9888} & \textbf{0.0132} & \textbf{3.99} \\
    & 0.95           & 0.0059 & 43.15 & 0.9802 & 0.0258 & 4.15 \\
    & 0.90           & 0.0062 & 43.06 & 0.9798 & 0.0269 & 12.14 \\
\midrule
Chest X-ray & 0.99 (default) & \textbf{0.0065} & \textbf{42.98} & \textbf{0.9885} & \textbf{0.0201} & \textbf{6.02} \\
            & 0.95           & 0.0127 & 42.02 & 0.9799 & 0.0394 & 6.26 \\
            & 0.90           & 0.0133 & 41.93 & 0.9791 & 0.0411 & 18.32 \\
\midrule
FFHQ & 0.99 (default) & \textbf{0.0109} & \textbf{39.16} & \textbf{0.9523} & \textbf{0.0301} & \textbf{9.17} \\
     & 0.95           & 0.0214 & 38.20 & 0.9431 & 0.0589 & 9.54 \\
     & 0.90           & 0.0223 & 38.11 & 0.9410 & 0.0615 & 27.92 \\
\bottomrule
\end{tabular}
\end{table}

\subsection{Patch Size}
\label{app:patch_ablation}

Table~\ref{tab:patch_ablation} compares patch sizes $8\times8$, $16\times16$, and $32\times32$. The $16\times16$ setting consistently achieves the best RMSE and FID across datasets. Smaller patches increase sequence length and reduce per-patch compressibility, while larger patches require more POD modes to retain the same energy, offsetting dimensionality reduction. Overall, $16\times16$ provides the best trade-off between locality, compression, and sequence length.
\begin{table}[htbp]
\centering
\caption{Effect of patch size on reconstruction quality. Lower is better for RMSE and FID; higher for PSNR and SSIM.}
\label{tab:patch_ablation}
\footnotesize
\setlength{\tabcolsep}{5pt}
\begin{tabular}{l l cccc}
\toprule
Dataset & Patch & RMSE & PSNR & SSIM & FID \\
\midrule
SST & $8\times8$   & 0.0040 & 42.93 & 0.9811 & 5.112 \\
    & $16\times16$ (default) & \textbf{0.0030} & \textbf{50.43} & \textbf{0.9888} & \textbf{3.986} \\
    & $32\times32$ & 0.0039 & 42.59 & 0.9836 & 4.011 \\
\midrule
Chest X-ray & $8\times8$   & 0.0078 & 38.75 & 0.9799 & 7.156 \\
            & $16\times16$ (default) & \textbf{0.0065} & \textbf{42.98} & \textbf{0.9885} & \textbf{6.015} \\
            & $32\times32$ & 0.0070 & 42.57 & 0.9822 & 6.589 \\
\midrule
FFHQ & $8\times8$   & 0.0185 & 37.95 & 0.9399 & 9.999 \\
     & $16\times16$ (default) & \textbf{0.0109} & \textbf{39.16} & \textbf{0.9523} & \textbf{9.169} \\
     & $32\times32$ & 0.0187 & 37.99 & 0.9455 & 9.530 \\
\bottomrule
\end{tabular}
\end{table}

\subsection{Denoiser Architecture: ViT vs.\ Per-Token MLP}
\label{app:mlp_ablation}
Table~\ref{tab:mlp_ablation} compares the ViT denoiser with a per-token MLP that processes patches independently. Removing cross-patch communication substantially degrades performance, increasing FID by $2.2$--$3.2\times$ and worsening RMSE across all datasets. This confirms that self-attention is important for modeling inter-patch dependencies.
\begin{table}[htbp]
\centering
\caption{ViT denoiser versus per-token MLP. Lower is better for RMSE, LPIPS, FID; higher for PSNR, SSIM.}
\label{tab:mlp_ablation}
\footnotesize
\setlength{\tabcolsep}{4pt}
\begin{tabular}{l l ccccc}
\toprule
Dataset & Denoiser & RMSE & PSNR & SSIM & LPIPS & FID \\
\midrule
SST & ViT & \textbf{0.0030} & \textbf{50.43} & \textbf{0.9888} & \textbf{0.0132} & \textbf{3.99} \\
    & MLP & 0.0698 & 37.72 & 0.9456 & 0.0682 & 12.59 \\
\midrule
Chest X-ray & ViT & \textbf{0.0065} & \textbf{42.98} & \textbf{0.9885} & \textbf{0.0201} & \textbf{6.02} \\
            & MLP & 0.0431 & 34.33 & 0.9555 & 0.0982 & 13.22 \\
\midrule
FFHQ & ViT & \textbf{0.0109} & \textbf{39.16} & \textbf{0.9523} & \textbf{0.0301} & \textbf{9.17} \\
     & MLP & 0.1182 & 30.07 & 0.9001 & 0.1011 & 28.57 \\
\bottomrule
\end{tabular}
\end{table}
%




\subsection{Effect of Ensemble Size}
\label{app:ensemble_size}

\begin{table}[htbp]
\centering
\caption{MACE and empirical coverage at selected nominal levels as a function of ensemble size $M$, averaged over all 2011 SST test days.}
\label{tab:ensemble_size}
\footnotesize
\begin{tabular}{lcccc}
\toprule
& $M=25$ & $M=50$ & $M=100$ & $M=200$ \\
\midrule
MACE & 0.0191 & 0.0091 & 0.0080 & 0.0081 \\
\midrule
50\% & 0.527 & 0.520 & 0.488 & 0.480 \\
70\% & 0.736 & 0.719 & 0.706 & 0.709 \\
90\% & 0.884 & 0.900 & 0.900 & 0.899 \\
95\% & 0.936 & 0.952 & 0.951 & 0.941 \\
99\% & 0.989 & 0.986 & 0.992 & 0.991 \\
\bottomrule
\end{tabular}
\end{table}
Table~\ref{tab:ensemble_size} reports MACE and selected empirical coverage levels for different ensemble sizes $M$, averaged over all 2011 SST test days. Calibration improves from $M=25$ to $M=50$, but gains beyond $M=100$ are marginal, indicating diminishing returns. We therefore use $M=100$ as the default, balancing ensemble inference cost and calibration accuracy. Notably, even $M=50$ achieves MACE below 0.01, suggesting that useful uncertainty estimates remain available under tighter compute budgets.
\subsection{Effect of Compression Level}
\label{app:compressionLevel}
To disentangle POD compression from POD structure, we evaluate a variant retaining all K=256 modes on SST. Despite identical ViT architecture and linear POD structure, this uncompressed variant yields RMSE 0.0046 and FID 5.18, worse than Patch-PODiff-ViT (RMSE 0.0030, FID 3.986) and comparable to DiT (RMSE 0.0040, FID 5.01). This confirms that variance-ordered truncation to the data's intrinsic dimensionality — not merely the linearity of the decoder — is the primary driver of reconstruction quality.


\section{Advection-Dominated Regime Analysis}
\label{app:advection}
\begin{table}[htbp]
\centering
\caption{Median $K_{99}$ and reconstruction error across datasets.}
\label{tab:advection}
\footnotesize
\setlength{\tabcolsep}{5pt}
\begin{tabular}{lccc}
\toprule
Dataset & Pe & $K_{99}$ & RMSE \\
\midrule
SST               & --       & 2   & 0.0030 \\
Chest X-ray       & --         & 11  & 0.0065 \\
FFHQ              & --         & 26  & 0.0109 \\
Adv.-dom.         & $O(10^6)$ & 150 & 0.0291 \\
\midrule
U-Net (adv.-dom.) & $O(10^6)$ & --  & 0.1985 \\
\bottomrule
\end{tabular}
\end{table}
We test Patch-PODiff-ViT in a strongly advection-dominated setting to validate the limitation in Section~\ref{sec:limitations}. Synthetic fields are generated from the 2D convection--diffusion equation with $\nu \in [10^{-5},10^{-4}]$ and velocities in $[-1,1]^2$, yielding $\mathrm{Pe}=O(10^6)$, on a $128\times128$ grid with $4\times$ downsampling.
Table~\ref{tab:advection} shows that this regime requires $K_{99}=150$ modes, far more than the main datasets, indicating much slower spectral decay. Patch-PODiff-ViT still achieves $\mathrm{RMSE}=0.0291$, improving over U-Net by $6.8\times$, but the larger latent size reduces the compression advantage. Thus, the method is most efficient when local patch spectra decay rapidly. The highly advective regimes may require adaptive bases, overlapping aggregation, or hybrid local--global covariance modeling.

\newpage
\section*{NeurIPS Paper Checklist}
\begin{enumerate}

\item {\bf Claims}
    \item[] Question: Do the main claims made in the abstract and introduction accurately reflect the paper's contributions and scope?
    \item[] Answer: \answerYes{} 
    \item[] Justification: The abstract and introduction summarize the proposed Patch-PODiff-ViT framework, its structured POD latent representation, uncertainty propagation, computational efficiency, and empirical evaluation. The claims are supported by the theoretical results, experiments, ablations, and limitations discussed in the paper.
    \item[] Guidelines:
    \begin{itemize}
        \item The answer \answerNA{} means that the abstract and introduction do not include the claims made in the paper.
        \item The abstract and/or introduction should clearly state the claims made, including the contributions made in the paper and important assumptions and limitations. A \answerNo{} or \answerNA{} answer to this question will not be perceived well by the reviewers. 
        \item The claims made should match theoretical and experimental results, and reflect how much the results can be expected to generalize to other settings. 
        \item It is fine to include aspirational goals as motivation as long as it is clear that these goals are not attained by the paper. 
    \end{itemize}

\item {\bf Limitations}
    \item[] Question: Does the paper discuss the limitations of the work performed by the authors?
    \item[] Answer: \answerYes{} 
    \item[] Justification: Section~\ref{sec:limitations} discusses limitations related to slow spectral decay, fixed POD bases, cross-patch covariance approximations, and discarded-mode uncertainty. Appendix~\ref{app:advection} further evaluates an advection-dominated regime where the compression advantage is reduced.
    \item[] Guidelines:
    \begin{itemize}
        \item The answer \answerNA{} means that the paper has no limitation while the answer \answerNo{} means that the paper has limitations, but those are not discussed in the paper. 
        \item The authors are encouraged to create a separate ``Limitations'' section in their paper.
        \item The paper should point out any strong assumptions and how robust the results are to violations of these assumptions (e.g., independence assumptions, noiseless settings, model well-specification, asymptotic approximations only holding locally). The authors should reflect on how these assumptions might be violated in practice and what the implications would be.
        \item The authors should reflect on the scope of the claims made, e.g., if the approach was only tested on a few datasets or with a few runs. In general, empirical results often depend on implicit assumptions, which should be articulated.
        \item The authors should reflect on the factors that influence the performance of the approach. For example, a facial recognition algorithm may perform poorly when image resolution is low or images are taken in low lighting. Or a speech-to-text system might not be used reliably to provide closed captions for online lectures because it fails to handle technical jargon.
        \item The authors should discuss the computational efficiency of the proposed algorithms and how they scale with dataset size.
        \item If applicable, the authors should discuss possible limitations of their approach to address problems of privacy and fairness.
        \item While the authors might fear that complete honesty about limitations might be used by reviewers as grounds for rejection, a worse outcome might be that reviewers discover limitations that aren't acknowledged in the paper. The authors should use their best judgment and recognize that individual actions in favor of transparency play an important role in developing norms that preserve the integrity of the community. Reviewers will be specifically instructed to not penalize honesty concerning limitations.
    \end{itemize}

\item {\bf Theory assumptions and proofs}
    \item[] Question: For each theoretical result, does the paper provide the full set of assumptions and a complete (and correct) proof?
    \item[] Answer: \answerYes{} 
    \item[] Justification: The assumptions and statements for Proposition~\ref{prop:patch_error} and Proposition~\ref{prop:patch_uq} are provided in the main text, with complete proofs given in Appendices~\ref{app:proof1} and~\ref{app:proof2}.
    \item[] Guidelines:
    \begin{itemize}
        \item The answer \answerNA{} means that the paper does not include theoretical results. 
        \item All the theorems, formulas, and proofs in the paper should be numbered and cross-referenced.
        \item All assumptions should be clearly stated or referenced in the statement of any theorems.
        \item The proofs can either appear in the main paper or the supplemental material, but if they appear in the supplemental material, the authors are encouraged to provide a short proof sketch to provide intuition. 
        \item Inversely, any informal proof provided in the core of the paper should be complemented by formal proofs provided in appendix or supplemental material.
        \item Theorems and Lemmas that the proof relies upon should be properly referenced. 
    \end{itemize}

    \item {\bf Experimental result reproducibility}
    \item[] Question: Does the paper fully disclose all the information needed to reproduce the main experimental results of the paper to the extent that it affects the main claims and/or conclusions of the paper (regardless of whether the code and data are provided or not)?
    \item[] Answer: \answerYes{} 
    \item[] Justification: The paper specifies the datasets, train/validation/test splits, preprocessing, POD basis construction, model architectures, training hyperparameters, sampling settings, and evaluation metrics. The processed SST benchmark fields will be released with documentation and access terms consistent with the original data providers, while X-ray and FFHQ use publicly available datasets.
    \item[] Guidelines:
    \begin{itemize}
        \item The answer \answerNA{} means that the paper does not include experiments.
        \item If the paper includes experiments, a \answerNo{} answer to this question will not be perceived well by the reviewers: Making the paper reproducible is important, regardless of whether the code and data are provided or not.
        \item If the contribution is a dataset and\slash or model, the authors should describe the steps taken to make their results reproducible or verifiable. 
        \item Depending on the contribution, reproducibility can be accomplished in various ways. For example, if the contribution is a novel architecture, describing the architecture fully might suffice, or if the contribution is a specific model and empirical evaluation, it may be necessary to either make it possible for others to replicate the model with the same dataset, or provide access to the model. In general. releasing code and data is often one good way to accomplish this, but reproducibility can also be provided via detailed instructions for how to replicate the results, access to a hosted model (e.g., in the case of a large language model), releasing of a model checkpoint, or other means that are appropriate to the research performed.
        \item While NeurIPS does not require releasing code, the conference does require all submissions to provide some reasonable avenue for reproducibility, which may depend on the nature of the contribution. For example
        \begin{enumerate}
            \item If the contribution is primarily a new algorithm, the paper should make it clear how to reproduce that algorithm.
            \item If the contribution is primarily a new model architecture, the paper should describe the architecture clearly and fully.
            \item If the contribution is a new model (e.g., a large language model), then there should either be a way to access this model for reproducing the results or a way to reproduce the model (e.g., with an open-source dataset or instructions for how to construct the dataset).
            \item We recognize that reproducibility may be tricky in some cases, in which case authors are welcome to describe the particular way they provide for reproducibility. In the case of closed-source models, it may be that access to the model is limited in some way (e.g., to registered users), but it should be possible for other researchers to have some path to reproducing or verifying the results.
        \end{enumerate}
    \end{itemize}

\item {\bf Open access to data and code}
    \item[] Question: Does the paper provide open access to the data and code, with sufficient instructions to faithfully reproduce the main experimental results, as described in supplemental material?
    \item[] Answer: \answerNo{} 
    \item[] Justification: ChestX-ray14 and FFHQ are publicly available. Upon acceptance, we will release the complete Patch-PODiff-ViT codebase, preprocessing pipelines, trained model configurations, evaluation scripts, and the processed SST benchmark fields needed to reproduce the main results. These will be hosted in a public repository or university-supported public link with documentation and reproduction instructions.
    \item[] Guidelines:
    \begin{itemize}
        \item The answer \answerNA{} means that paper does not include experiments requiring code.
        \item Please see the NeurIPS code and data submission guidelines (\url{https://neurips.cc/public/guides/CodeSubmissionPolicy}) for more details.
        \item While we encourage the release of code and data, we understand that this might not be possible, so \answerNo{} is an acceptable answer. Papers cannot be rejected simply for not including code, unless this is central to the contribution (e.g., for a new open-source benchmark).
        \item The instructions should contain the exact command and environment needed to run to reproduce the results. See the NeurIPS code and data submission guidelines (\url{https://neurips.cc/public/guides/CodeSubmissionPolicy}) for more details.
        \item The authors should provide instructions on data access and preparation, including how to access the raw data, preprocessed data, intermediate data, and generated data, etc.
        \item The authors should provide scripts to reproduce all experimental results for the new proposed method and baselines. If only a subset of experiments are reproducible, they should state which ones are omitted from the script and why.
        \item At submission time, to preserve anonymity, the authors should release anonymized versions (if applicable).
        \item Providing as much information as possible in supplemental material (appended to the paper) is recommended, but including URLs to data and code is permitted.
    \end{itemize}

\item {\bf Experimental setting/details}
    \item[] Question: Does the paper specify all the training and test details (e.g., data splits, hyperparameters, how they were chosen, type of optimizer) necessary to understand the results?
    \item[] Answer: \answerYes{} 
    \item[] Justification: Sections~\ref{sec:datasets}--\ref{sec:metrics} describe the datasets, preprocessing, baselines, optimizer, learning rate, diffusion schedule, sampling steps, ensemble size, and evaluation metrics. Additional implementation details and ablations are provided in the appendices.
    \item[] Guidelines:
    \begin{itemize}
        \item The answer \answerNA{} means that the paper does not include experiments.
        \item The experimental setting should be presented in the core of the paper to a level of detail that is necessary to appreciate the results and make sense of them.
        \item The full details can be provided either with the code, in appendix, or as supplemental material.
    \end{itemize}

\item {\bf Experiment statistical significance}
    \item[] Question: Does the paper report error bars suitably and correctly defined or other appropriate information about the statistical significance of the experiments?
    \item[] Answer: \answerYes{} 
    \item[] Justification: Reconstruction metrics are reported as mean and standard deviation across three seed runs in Appendix~\ref{app:extended_results}. CRPS is also reported with standard deviation, and empirical coverage is reported across multiple nominal levels.
    \item[] Guidelines:
    \begin{itemize}
        \item The answer \answerNA{} means that the paper does not include experiments.
        \item The authors should answer \answerYes{} if the results are accompanied by error bars, confidence intervals, or statistical significance tests, at least for the experiments that support the main claims of the paper.
        \item The factors of variability that the error bars are capturing should be clearly stated (for example, train/test split, initialization, random drawing of some parameter, or overall run with given experimental conditions).
        \item The method for calculating the error bars should be explained (closed form formula, call to a library function, bootstrap, etc.)
        \item The assumptions made should be given (e.g., Normally distributed errors).
        \item It should be clear whether the error bar is the standard deviation or the standard error of the mean.
        \item It is OK to report 1-sigma error bars, but one should state it. The authors should preferably report a 2-sigma error bar than state that they have a 96\% CI, if the hypothesis of Normality of errors is not verified.
        \item For asymmetric distributions, the authors should be careful not to show in tables or figures symmetric error bars that would yield results that are out of range (e.g., negative error rates).
        \item If error bars are reported in tables or plots, the authors should explain in the text how they were calculated and reference the corresponding figures or tables in the text.
    \end{itemize}

\item {\bf Experiments compute resources}
    \item[] Question: For each experiment, does the paper provide sufficient information on the computer resources (type of compute workers, memory, time of execution) needed to reproduce the experiments?
    \item[] Answer: \answerYes{} 
    \item[] Justification: The paper reports hardware type, training time, peak GPU memory, parameter count, per-sample inference time, and ensemble inference time for the main diffusion models in Section~\ref{sec:efficiency_results}. These values provide the compute requirements needed to reproduce the main experiments.
    \item[] Guidelines:
    \begin{itemize}
        \item The answer \answerNA{} means that the paper does not include experiments.
        \item The paper should indicate the type of compute workers CPU or GPU, internal cluster, or cloud provider, including relevant memory and storage.
        \item The paper should provide the amount of compute required for each of the individual experimental runs as well as estimate the total compute. 
        \item The paper should disclose whether the full research project required more compute than the experiments reported in the paper (e.g., preliminary or failed experiments that didn't make it into the paper). 
    \end{itemize}
    
\item {\bf Code of ethics}
    \item[] Question: Does the research conducted in the paper conform, in every respect, with the NeurIPS Code of Ethics \url{https://neurips.cc/public/EthicsGuidelines}?
    \item[] Answer: \answerYes{} 
    \item[] Justification: The research uses established public or licensed datasets, does not involve new human-subject data collection, and does not deploy the proposed model in a real-world decision-making system. We have reviewed the NeurIPS Code of Ethics and believe the work conforms to it.
    \item[] Guidelines:
    \begin{itemize}
        \item The answer \answerNA{} means that the authors have not reviewed the NeurIPS Code of Ethics.
        \item If the authors answer \answerNo, they should explain the special circumstances that require a deviation from the Code of Ethics.
        \item The authors should make sure to preserve anonymity (e.g., if there is a special consideration due to laws or regulations in their jurisdiction).
    \end{itemize}

\item {\bf Broader impacts}
    \item[] Question: Does the paper discuss both potential positive societal impacts and negative societal impacts of the work performed?
    \item[] Answer: \answerYes{} 
    \item[] Justification: The work may support uncertainty-aware scientific and medical super-resolution by providing calibrated spatial confidence maps. However, generated high-resolution outputs should not be treated as direct observations in high-stakes settings and should require domain validation, uncertainty reporting, and expert oversight.
    \item[] Guidelines:
    \begin{itemize}
        \item The answer \answerNA{} means that there is no societal impact of the work performed.
        \item If the authors answer \answerNA{} or \answerNo, they should explain why their work has no societal impact or why the paper does not address societal impact.
        \item Examples of negative societal impacts include potential malicious or unintended uses (e.g., disinformation, generating fake profiles, surveillance), fairness considerations (e.g., deployment of technologies that could make decisions that unfairly impact specific groups), privacy considerations, and security considerations.
        \item The conference expects that many papers will be foundational research and not tied to particular applications, let alone deployments. However, if there is a direct path to any negative applications, the authors should point it out. For example, it is legitimate to point out that an improvement in the quality of generative models could be used to generate Deepfakes for disinformation. On the other hand, it is not needed to point out that a generic algorithm for optimizing neural networks could enable people to train models that generate Deepfakes faster.
        \item The authors should consider possible harms that could arise when the technology is being used as intended and functioning correctly, harms that could arise when the technology is being used as intended but gives incorrect results, and harms following from (intentional or unintentional) misuse of the technology.
        \item If there are negative societal impacts, the authors could also discuss possible mitigation strategies (e.g., gated release of models, providing defenses in addition to attacks, mechanisms for monitoring misuse, mechanisms to monitor how a system learns from feedback over time, improving the efficiency and accessibility of ML).
    \end{itemize}
    
\item {\bf Safeguards}
    \item[] Question: Does the paper describe safeguards that have been put in place for responsible release of data or models that have a high risk for misuse (e.g., pre-trained language models, image generators, or scraped datasets)?
    \item[] Answer: \answerNA{} 
    \item[] Justification: The paper does not release a high-risk foundation model, scraped dataset, or general-purpose image generator. FFHQ is used only as a standard benchmark for reconstruction evaluation, and the proposed method is not presented for identity recognition, surveillance, or unrestricted image generation.
    \item[] Guidelines:
    \begin{itemize}
        \item The answer \answerNA{} means that the paper poses no such risks.
        \item Released models that have a high risk for misuse or dual-use should be released with necessary safeguards to allow for controlled use of the model, for example by requiring that users adhere to usage guidelines or restrictions to access the model or implementing safety filters. 
        \item Datasets that have been scraped from the Internet could pose safety risks. The authors should describe how they avoided releasing unsafe images.
        \item We recognize that providing effective safeguards is challenging, and many papers do not require this, but we encourage authors to take this into account and make a best faith effort.
    \end{itemize}

\item {\bf Licenses for existing assets}
    \item[] Question: Are the creators or original owners of assets (e.g., code, data, models), used in the paper, properly credited and are the license and terms of use explicitly mentioned and properly respected?
    \item[] Answer: \answerYes{} 
    \item[] Justification: Existing datasets and methods are cited in the paper. ChestX-ray14 is used under the NIH
Clinical Center terms with required attribution; FFHQ is used under its Creative Commons BY-NC-SA 4.0 license; ROMS is distributed under an MIT/X-style license; and ACCESS-S2 data are used under Bureau of Meteorology data-access terms. The processed SST benchmark fields will be released with appropriate attribution, documentation, and access terms consistent with the ROMS and ACCESS-S2 data providers.
    \item[] Guidelines:
    \begin{itemize}
        \item The answer \answerNA{} means that the paper does not use existing assets.
        \item The authors should cite the original paper that produced the code package or dataset.
        \item The authors should state which version of the asset is used and, if possible, include a URL.
        \item The name of the license (e.g., CC-BY 4.0) should be included for each asset.
        \item For scraped data from a particular source (e.g., website), the copyright and terms of service of that source should be provided.
        \item If assets are released, the license, copyright information, and terms of use in the package should be provided. For popular datasets, \url{paperswithcode.com/datasets} has curated licenses for some datasets. Their licensing guide can help determine the license of a dataset.
        \item For existing datasets that are re-packaged, both the original license and the license of the derived asset (if it has changed) should be provided.
        \item If this information is not available online, the authors are encouraged to reach out to the asset's creators.
    \end{itemize}

\item {\bf New assets}
    \item[] Question: Are new assets introduced in the paper well documented and is the documentation provided alongside the assets?
    \item[] Answer: \answerNA{} 
    \item[] Justification: The paper does not introduce a new dataset or benchmark asset. The codebase will be made available upon acceptance, with full documentation for reproducing the Patch-PODiff-ViT pipeline across all three datasets. Access to the SST benchmark fields will follow the pathway described in Item 5. 
    \item[] Guidelines:
    \begin{itemize}
        \item The answer \answerNA{} means that the paper does not release new assets.
        \item Researchers should communicate the details of the dataset\slash code\slash model as part of their submissions via structured templates. This includes details about training, license, limitations, etc. 
        \item The paper should discuss whether and how consent was obtained from people whose asset is used.
        \item At submission time, remember to anonymize your assets (if applicable). You can either create an anonymized URL or include an anonymized zip file.
    \end{itemize}

\item {\bf Crowdsourcing and research with human subjects}
    \item[] Question: For crowdsourcing experiments and research with human subjects, does the paper include the full text of instructions given to participants and screenshots, if applicable, as well as details about compensation (if any)? 
    \item[] Answer: \answerNA{} 
    \item[] Justification: The paper does not involve crowdsourcing, user studies, or new human-subject data collection.
    \item[] Guidelines:
    \begin{itemize}
        \item The answer \answerNA{} means that the paper does not involve crowdsourcing nor research with human subjects.
        \item Including this information in the supplemental material is fine, but if the main contribution of the paper involves human subjects, then as much detail as possible should be included in the main paper. 
        \item According to the NeurIPS Code of Ethics, workers involved in data collection, curation, or other labor should be paid at least the minimum wage in the country of the data collector. 
    \end{itemize}

\item {\bf Institutional review board (IRB) approvals or equivalent for research with human subjects}
    \item[] Question: Does the paper describe potential risks incurred by study participants, whether such risks were disclosed to the subjects, and whether Institutional Review Board (IRB) approvals (or an equivalent approval/review based on the requirements of your country or institution) were obtained?
    \item[] Answer: \answerNA{} 
    \item[] Justification: The paper does not involve new human-subject research, crowdsourcing, or interaction with study participants. ChestX-ray14 and FFHQ are pre-existing datasets used as benchmarks.
    \item[] Guidelines:
    \begin{itemize}
        \item The answer \answerNA{} means that the paper does not involve crowdsourcing nor research with human subjects.
        \item Depending on the country in which research is conducted, IRB approval (or equivalent) may be required for any human subjects research. If you obtained IRB approval, you should clearly state this in the paper. 
        \item We recognize that the procedures for this may vary significantly between institutions and locations, and we expect authors to adhere to the NeurIPS Code of Ethics and the guidelines for their institution. 
        \item For initial submissions, do not include any information that would break anonymity (if applicable), such as the institution conducting the review.
    \end{itemize}

\item {\bf Declaration of LLM usage}
    \item[] Question: Does the paper describe the usage of LLMs if it is an important, original, or non-standard component of the core methods in this research? Note that if the LLM is used only for writing, editing, or formatting purposes and does \emph{not} impact the core methodology, scientific rigor, or originality of the research, declaration is not required.
    \item[] Answer: \answerNA{} 
    \item[] Justification: The core method does not use LLMs as a model component, experimental tool, or source of scientific results.
    \item[] Guidelines:
    \begin{itemize}
        \item The answer \answerNA{} means that the core method development in this research does not involve LLMs as any important, original, or non-standard components.
        \item Please refer to our LLM policy in the NeurIPS handbook for what should or should not be described.
    \end{itemize}

\end{enumerate}

\end{document}